\documentclass[runningheads]{llncs}

\usepackage[T1]{fontenc}
\usepackage{graphicx, amsmath, amssymb, xspace, scalerel,thm-restate}

\usepackage{amsthm}
\usepackage{lineno} 

\usepackage{tikz}
\usetikzlibrary{arrows.meta, positioning}

\definecolor{mygray}{RGB}{224,224,224}

\providecommand{\A}{A}
\providecommand{\F}{\mathcal{F}}
\providecommand{\G}{\mathcal{G}}

\providecommand{\C}{\mathcal{C}}

\providecommand{\att}{\rightarrowtail}
\providecommand{\natt}{\not\rightarrowtail}
\providecommand{\nat}{\mathbb{N}}

\providecommand{\stg}{\textit{stg}}

\providecommand{\gr}{\textit{gr}}
\providecommand{\na}{\textit{na}}

\providecommand{\cf}{\textit{cf}}

\providecommand{\cft}{\textit{cf2}\xspace}
\providecommand{\cfo}{\textit{cf1.5}\xspace}

\providecommand{\icft}{\textit{tfcf2}\xspace}
\providecommand{\stgt}{\textit{stg2}\xspace}
\providecommand{\stgo}{\textit{stg1.5}\xspace}
\providecommand{\istgt}{{\textit{tfstg2}}\xspace}
\providecommand{\compordinal}{component ordinal\xspace}

\renewcommand{\hat}{\widehat}

\DeclareMathOperator{\SCC}{SCC}

\begin{document}
\title{SCC-recursiveness in infinite argumentation (extended version)}
%
%
\author{Uri Andrews\inst{1}\orcidID{0000-0002-4653-7458}
	 \and
Luca San Mauro\inst{2}\orcidID{0000-0002-3156-6870}
}
\authorrunning{U. Andrews, L. San Mauro}
%
\institute{University of Wiscosin--Madison, Madison, WI 53706, USA
\email{andrews@math.wisc.edu}\\
\url{http://math.wisc.edu/$\sim$andrews} \and
University of Bari, Bari, Italy
\email{luca.sanmauro@gmail.com}\\
\url{https://www.lucasanmauro.com/}}
\maketitle              
\begin{abstract}

Argumentation frameworks (AFs) are a foundational tool in artificial intelligence for modeling structured reasoning and conflict. 
SCC-recursiveness is a well-known design principle in which the evaluation of arguments is decomposed according to the strongly connected components (SCCs) of the attack graph, proceeding recursively from "higher" to "lower" components.  
While SCC-recursive semantics such as \cft and \stgt have proven effective for finite AFs, Baumann and Spanring showed the failure of SCC-recursive semantics to generalize reliably to infinite AFs due to issues with well-foundedness. 

We propose two approaches to extending SCC-recursiveness to the infinite setting.  
We systematically evaluate these semantics using Baroni and Giacomin’s established criteria, showing in particular that directionality fails in general. We then examine these semantics' behavior in finitary frameworks, where we find some of our semantics satisfy directionality.
These results advance the theory of infinite argumentation and lay the groundwork for reasoning systems capable of handling unbounded or evolving domains.

\end{abstract}

\section{Introduction}

Formal argumentation is a major research area in artificial intelligence that models reasoning and debate by representing arguments and their interactions in a structured form. A foundational concept is that of \emph{Argumentation Framework} (AF), where arguments are treated as abstract elements and their conflicts are captured by a binary attack relation, denoted $a\att b$, indicating that accepting argument $a$ provides a reason to reject argument $b$. Dung \cite{dung1995acceptability} introduced several \emph{semantics}---that is, formal criteria to identify acceptable collections of arguments called \emph{extensions}---to accomodate various reasoning contexts. A known challenge in Dung's approach arises in scenarios involving odd-length cycles: e.g., $a\att b\att c\att a$, which can lead to counterintuitive outcomes. Baroni and Giacomin \cite{BaroniGiacomin03} proposed the  \cf2 semantics which yields more intuitive results in the presence of such cycles. This idea was further developed by Baroni, Giacomin, and Guida \cite{BaroniGiacominGuida}  into a  general framework of \emph{SCC-recursiveness}, where the acceptability of arguments is determined recursively based on the structure of the \emph{strongly connected components} (SCCs) of the argumentation framework's attack graph. These approaches have proven particularly successful for finite argumentation frameworks.

Infinite AFs, where an infinite number of arguments is considered, are important because many real-world domains involve potentially unbounded or dynamically generated arguments. Such arguments may arise from, e.g., inductive definitions, logical deductions, recursive structures, or ongoing information streams. In recent years, interest in infinite AFs has grown (see, e.g., \cite{verheij2003deflog,caminada2014grounded-inf,baumann2015infinite,baroni2012computing,bistarelli2021weighted}); our previous work contributes to this line of research by exploring the computational complexity of natural reasoning problems in infinite AFs (see \cite{andrews2024FCR,andrews2024NMR}). More broadly, the study of infinite AFs often involves analyzing the limiting behavior of semantics---understanding which properties persist in unbounded contexts and designing algorithms that converge toward correct conclusions about infinite frameworks.


Baumann and Spanring \cite{BS-unrestricted} identified a fundamental issue with SCC-recur\-siveness, particularly in the context of \cft and \stgt semantics, when applied to infinite AFs. Specifically, in the infinite setting, the recursive definition may be ill-founded, rendering the  semantics themselves not well-defined. Consequently, Baroni and Giacomin’s semantics---despite their naturalness and elegance in the finite case---do not reliably extend to the infinite domain.

 In this paper, we propose two solutions to address this problem.
The first approach involves extending the recursion \emph{transfinitely}, which we see (Remark \ref{rmk:Equivalence With Gaggl} below) is equivalent to alternative characterizations previously proposed by Gaggl and Woltran \cite{Gaggl-Woltran-cf2} and Dvo\v r\'ak and Gaggl \cite{Dvorak-Gaggl-stg2} in the finite setting. The second approach is to introduce two new semantics, which we term \cfo and \stgo (Definition \ref{def:1.5 semantics}). These are defined by modifying the standard definitions of \cft and \stgt by preserving the initial notion of strongly connected components throughout the recursive process, rather than redefining them at each step. 

We analyze these new semantics in light of the evaluation criteria proposed by Baroni and Giacomin \cite{BaroniGiacomin}, identifying in each case exactly which of these criteria are preserved in the infinite setting  (Table \ref{table:properties}). Additionally, we use our semantics to revisit the question raised by Baumann and Spanring \cite{baumann2015infinite} concerning whether  every \emph{finitary} AF---that is, an AF in which  each argument is attacked by only finitely many others---necessarily admits a \cft or \stgt extension.  Finally, we provide a deeper exploration of the finitary case, proving that   \cfo and \stgo are, in general, more well-behaved than the transfinite extensions of \cft or \stgt (Table  \ref{table:finitary-properties}).


\section{Background}

\subsection{Logic background}
Recall that the ordinals extend the counting
numbers into the infinite. We will use the following fundamental properties of the collection of ordinals:
\begin{itemize}
    \item The collection of ordinals is linearly ordered, and every non-empty subset of the collection of ordinals has a least element.
\end{itemize}

A well-ordering is order of an ordinal. That is, it is a linear order so that every non-empty subset has a least element. Recall that the ordinal $\omega$ is the least infinite ordinal. Also recall  that every ordinal $\alpha$ is either the
successor of another ordinal $\beta$, i.e., $\alpha=\beta+1$, or is a limit ordinal, i.e., $\alpha =sup(Y)$, which is the
supremum of the ordinals in the set $Y$. The ordinal $\omega$ is the first limit ordinal, and is the supremum of the finite
ordinals, i.e., the natural numbers.

We let $\nat^{<\nat}$ represent the collection of \emph{strings}, i.e., finite sequences, of natural numbers and $\nat^\nat$ represent infinite sequences of natural numbers. 
A string $\sigma\in \nat^{<\nat}$ is a \emph{prefix} of a string $\tau\in \nat^{<\nat}$ or of a sequence $\tau \in \nat^{\nat}$, written $\sigma\preceq \tau$, if there is some $\rho\in \nat^{<\nat}$ or $\rho \in \nat^{\nat}$ so that $\tau = \sigma^\smallfrown\rho$, i.e, $\tau$ is the concatenation of $\sigma$ with $\rho$. Two strings $\sigma$ and $\tau$ are comparable if $\sigma\preceq \tau$ or $\tau \preceq \sigma$.
A \emph{tree} is a subset of $\nat^{<\nat}$ which is closed under prefixes. A \emph{path} $\pi$ through a tree $T$ is an element of $\nat^\nat$ so that every prefix of $\pi$ is in $T$. 
Note that since our trees may be infinitely branching, even trees which contain arbitrarily long strings (i.e., are of infinite \emph{height}) may not have paths. Consider for example, the tree containing the empty string and all strings $i^\smallfrown\sigma$ where the length of $\sigma$ is $\leq i$.

Finally, the proofs of Theorems \ref{thm:finitary existence cf2} and \ref{thm:finitary existence stg} will employ the compactness theorem for propositional logic \cite[\S XI.4, Theorem 4.5]{EFT}. That is, a theory $T$ of propositional logic is consistent if and only if every finite subset of $T$ is consistent.
\subsection{Argumentation theoretic background}\label{sec:AF Background}

We briefly review some key concepts of Dung-style   argumentation theory (for an overview of this area, we refer the reader to the surveys \cite{baroni2009semantics,dunne2009complexity}). 

An \emph{argumentation framework} (AF) $\F$  is a pair $(A_\F,R_\F)$ consisting of  a set $A_\F$ of arguments and an attack relation $R_\F\subseteq A_\F\times A_\F$. If some argument $a$ attacks some argument $b$, 
 we often write $a\att b$ instead of $(a, b)\in R_\F$.   Collections of arguments $S\subseteq A_\F$ are called \emph{extensions}. For any extension $S$, denote by $\F\restriction_S$ the sub-framework of $\F$ with respect to $S$: i.e.,  $\F\restriction_S = (A_\F \cap S, R_\F \cap (S \times S))$.

 For an extension $S$, the symbols $S^+$ and $S^-$ denote, respectively, the arguments that $S$ attacks and the arguments that attack $S$: 
 \[
 S^+=\{x : (\exists y \in  S)(y \att x)\};
 S^-=\{x : (\exists y \in  S)(x \att y) \}.
 \]
 $S$ \emph{defends} an argument $a$, if  any argument that attacks $a$ is attacked by some argument in $S$ (i.e., $\{a\}^-\subseteq S^+$). The \emph{range} $S^{\oplus}$ of $S$ as $S\cup S^+$.  The \emph{characteristic function} of $\F$ is the mapping $f_\F$ which sends subsets of $A_\F$ to subsets of $A_\F$ via 
 $f_\F(S) := \{x :  x \text{ is defended by $S$}\}$. 
Most AFs investigated in this paper are infinite. 
An AF
 $\F$ is \emph{finitary} if  $\{x\}^-$ is finite for all $x\in A_\F$.

A \emph{semantics} $\sigma$  assigns to every AF $\F$ a set of extensions $\sigma(\F)$ which are  deemed as acceptable.  Several semantics have been proposed and analyzed.  Four prominent semantics are relevant here: conflict-free, grounded, naive, stage (abbreviated by $\cf, \gr, \na, \stg,$ respectively). 
First, denote by $\cf(\F)$ the collection of extensions of $\F$ which are \emph{conflict-free}: i.e., $S\in \cf(\F)$ iff  $a \natt b$, for all $a,b\in S$. Then, for $S\in \cf(\F)$,
 \begin{itemize}
 \item $S\in \na(\F)$ iff there is no $T\in \cf(\F)$ with  $T\supsetneq S$;
 \item $S \in \gr(\F)$ iff  $S=f_F(S)$ and there is no $T\subsetneq S$ with $T=f_F(T)$;
 \item $S \in \stg(\F)$ iff there is no  $T \in  \cf(\F)$ with $S^{\oplus}\subsetneq {T}^{\oplus}$.
\end{itemize}

We recall that for each AF $\F$, the grounded semantics yields a unique extension, which is the least fixed-point of the characteristic
function $f_\F$. Observe that every stage extension is naive.

Finally, we give the recursive definition of $\cft$ and $\stgt$ semantics---—the main objects of study in this paper---which are based on the graph-theoretic notion of strongly connected component.

\begin{definition}
    For an AF $\F$, the \emph{strongly connected component} of   $a\in A_\F$, written $\SCC(a)$ is the set of arguments $b\in A_\F$ so that there exists a directed path from $a$ to $b$ (i.e., a sequence $c_0\att c_1 \att c_2 \dots \att c_n$ so that $c_0=a$ and $c_n=b$) and there exists a directed path from $b$ to $a$.  
\end{definition}

We denote by $\SCC(\F)$ the strongly connected components of $\F$; note that $\SCC(\F)$ is a partition of $A_F$. In particular, $b\in \SCC(a)$ iff $a\in \SCC(b)$.

An important feature of the decomposition of an AF into strongly connected components is that the resulting graph---where each SCC is treated as a single node---is acyclic; in other words, the attack relation induces a partial order over the SCCs. The SCC-recursive schema (see Baroni, Giacomin, and Guida \cite{BaroniGiacominGuida}) employs this property to define a recursive procedure that incrementally builds extensions by processing the SCCs in accordance with their partial order. In Definition \ref{def2}, $D_S(\SCC(b))$ is the set of elements in the strongly connected component of $b$ that are not invalidated by an argument in $S$ belonging to an SCC that precedes 
$\SCC(b)$ in the partial order.


\begin{definition}\label{def2}
    For $X,S\subseteq A_\F$, we let 
    \[
    D_{S}(X)=\{b\in X : (\exists a\in S\smallsetminus X)(a\att b)\}.
    \]
\end{definition}


The formal definition of the $\cft$ semantics is as follows:

\begin{definition}[Baroni-Giacomin \cite{BaroniGiacomin03}]\label{def:cf2}
Let $\F=(A_\F,R_\F)$ be an AF and $S\subseteq A_\F$. Then, $S \in \cft(\F)$ iff:
    \begin{itemize}
        \item $|\SCC(\F)|=1$ and $S \in \na(F)$;
        \item or, for each $X\in \SCC(\F)$, $(S \cap X) \in \cft\left(\F\restriction_{X \smallsetminus D_S(X)}\right)$.
    \end{itemize}

\end{definition}

Dvo\v r\'ak and Gaggl give the \stgt semantics similarly by using the $\stg$ semantics in place of the $\na$ semantics:

\begin{definition}[Dvo\v r\'ak-Gaggl \cite{Dvorak-Gaggl-stg2}]\label{def:stg2}
    Let $\F = (A_\F, R_\F)$ and let $S \subseteq A_\F$. Then, $S\in \stgt(\F)$ iff:
    \begin{itemize}
        \item $|\SCC(\F)|=1$ and $S \in \stg(F)$;
        \item or, for each $X\in \SCC(\F)$, $(S \cap X) \in \stgt\left(\F\restriction_{X \smallsetminus D_S(X)}\right)$.
    \end{itemize}
\end{definition}

For finite AFs, the above definitions are well-defined---that is, the recursion always terminates---since, if $\F$ has more than one strongly connected component, then $\vert X\smallsetminus D_S(X)\vert < \vert A_F \vert$. In contrast, the situation becomes more delicate in the infinite setting, as illustrated by the following example.

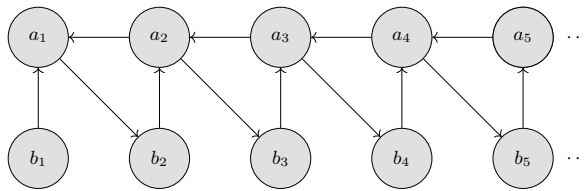
\begin{figure}\centering
\scalebox{0.8}{
\begin{tikzpicture}[->=Stealth, state/.style={circle, draw, fill=mygray, minimum size=1cm}]

    \node[state, 
    ] (a1) {$a_1$};
    \node[state, 
    ] (a2) [right=of a1] {$a_2$};
    \node[state, 
    ] (a3) [right=of a2] {$a_3$};
    \node[state, 
    ] (a4) [right=of a3] {$a_4$};
    \node[state, 
    ] (a5) [right=of a4] {$a_5$};
    \node[state, 
    ] (a5) [right=of a4] {$a_5$};
   \node[draw=none] (ellipsis2)[right=1mm of a5] {$\cdots$};

    \node[state, 
    ] (b1) [below=of a1] {$b_1$};
    \node[state, 
    ] (b2) [below=of a2] {$b_2$};
    \node[state, 
    ] (b3) [below=of a3] {$b_3$};
    \node[state, 
    ] (b4) [below=of a4] {$b_4$};
    \node[state, 
    ] (b5) [below=of a5] {$b_5$};
    \node[draw=none] (ellipsis2)[right=1mm of b5] {$\cdots$};

\draw[->] (a2) edge (a1);
\draw[->] (a3) edge (a2);
\draw[->] (a4) edge (a3);
\draw[->] (a5) edge (a4);

\draw[->] (a1) edge (b2);
\draw[->] (a2) edge (b3);
\draw[->] (a3) edge (b4);
\draw[->] (a4) edge (b5);

\draw[->] (b1) edge (a1);
\draw[->] (b2) edge (a2);
\draw[->] (b3) edge (a3);
\draw[->] (b4) edge (a4);
\draw[->] (b5) edge (a5);



\end{tikzpicture}
}

\caption{An example from \cite{BS-unrestricted} of an infinite AF where the \cft and \stgt semantics are not well-defined.}
\label{BS-example}

\end{figure}

\begin{example}[Baumann-Spanring \cite{BS-unrestricted}]
The example in Figure \ref{BS-example} illustrates an AF $\F$ for which the notion of a $\cft$ or $\stgt$ extension is \emph{not} well-defined. For concreteness, we consider the case of  $\cft$ (the case of $\stgt$ is analogous). Observe that $\F$ consists of two strongly connected components: $\{b_1\}$ and $A_\F \smallsetminus \{b_1\}$. Now, let $B=\{b_i : i\in \nat\}$. By definition,  
    \[    B\in \cft \Leftrightarrow B\smallsetminus \{b_1\}\in \cf2(\F\restriction_{A_\F\smallsetminus \{b_1,a_1\}}).  
    \]
    The restricted framework $\F\restriction_{A_\F\smallsetminus \{b_1,a_1\}}$ again contains only two strongly connected components, one of which is the singleton $\{b_2\}$.
    Iterating this process, we repeatedly remove pairs $b_i$ and $a_i$ from the frameworks.  However, we never reach a stage where only a single strongly connected component remains. As a result, the recursion underlying the definition of $\cft$ does not terminate, and such a semantics is not well-defined for $\F$.
\end{example}

\section{SCC-recursiveness in the infinite setting}\label{sec:defining semantics}
In this section, we propose two novel approaches to extending the concept of SCC-recursive semantics to the infinite setting. The first approach relaxes the requirement that the recursion must terminate after finitely many steps, instead allowing the process to continue transfinitely along the ordinals.

\subsection{Use transfinite recursion}
We begin by inductively defining the strongly connected component of an element $a\in A_\F$ at  ordinal stages.

\begin{definition}
    Let $\F$ be an AF and let $S\subseteq A_\F$. For $a\in A_\F$ and all ordinals $\alpha$, define $C_S^\alpha(a)$ inductively as follows: 
    \begin{itemize}
        \item   $C_S^0(a)=\SCC(a)$;
        \item $C_S^{\alpha+1}(a)$ is the strongly connected component of $a$ in $C_S^\alpha(a)\smallsetminus D_S(C_S^\alpha(a))$;
        \item finally, for limit ordinals $\lambda$, $C_S^\lambda(a)$ is the strongly connected component of $a$ in $\bigcap_{\alpha<\lambda}C_S^\alpha(a)$.
    \end{itemize}   
        For $a\in A_\F$ and $S\subseteq A_\F$, we let the \emph{\compordinal} of $a$ over $S$, written $\alpha_S(a)$,  be the least $\alpha$ so that either $a\notin C_S^\alpha(a)$ or $C_S^{\alpha+1}(a)=C_S^\alpha(a)$.
\end{definition}

We are now ready to define the semantics that extend $\cft$ and 
$\stgt$ transfinitely; we denote these by 
$\icft$ and $\istgt$, respectively.

\begin{definition}\label{def:icft}\label{def:istgt}
Let $\F=(A_\F,R_\F)$ and let $S\in \cf(\F)$. Then, 

\begin{itemize}
    \item $S\in\icft(\F)$ iff, for each $a\in \F$, either $a\notin C_S^{\alpha_S(a)}(a)$ or $S\cap C_S^{\alpha_S(a)}(a)$ is a naive extension of $\F\restriction_{C_S^{\alpha_S(a)}(a)}$;
    \item $S\in \istgt(\F)$ iff, for each $a\in \F$, either $a\notin C_S^{\alpha_S(a)}(a)$ or $S\cap C_S^{\alpha_S(a)}(a)$ is a stage extension of $\F\restriction_{C_S^{\alpha_S(a)}(a)}$.
\end{itemize}
\end{definition}


\begin{remark}
Following the same pattern but using the semantics $\sigma$ instead of naive or stage, one may define, for any base semantics $\sigma$, a new semantics \textit{tf$\sigma$2} which transfinitely extends the SCC-recursive semantics $\sigma2$. However, in this paper we restrict our attention to the specific cases of $\icft$ and $\istgt$.\footnote{We note with this general definition, \icft should be called \textit{tfna2}, but we choose to follow the standard terminology from the finite setting.}
\end{remark}


Observe that in the example from Figure \ref{BS-example}, $\{b_i : i\in \nat\}$ is both a \icft and  a \istgt extension. For each argument $b_i$, at the ordinal $\alpha=i$, its strongly connected component stabilizes to the singleton $\{b_i\}$, and the restriction of $B$ to this component forms a naive and a stage extension, respectively. Similarly, for each argument $a_i$, at the ordinal $i+1$, we see that $a_i\notin C_B^{i+1}(a_i)$, and thus it is correctly excluded from the extension.

\begin{remark}\label{rmk:Equivalence With Gaggl}
    Gaggl and Woltran \cite[Theorem 3.11]{Gaggl-Woltran-cf2}, as well as  Dvo\v r\'ak and Gaggl \cite[Proposition 3.2]{Dvorak-Gaggl-stg2}, provided alternative characterizations of the $\cft$ and $\stgt$ semantics in the finite setting. Their approach is based on the least fixed point of an operator $\Delta_{\F,S}$. Notably, the existence of this least fixed point does not depend on the finiteness of the argumentation framework. In fact, it turns out 
    that their characterizations, if considered in the infinite setting, are equivalent to the definition of the \icft and \istgt semantics given in this paper. A proof of this equivalence is provided in Appendix \ref{sec:equivalence}.
\end{remark}

One downside of working with \icft and \istgt is that, for countable AFs $\F$, there may be $a\in \F$ and $S\subseteq A_\F$ whose \compordinal  $\alpha_S(a)$ is an arbitrarily large countable ordinal. See Appendix \ref{sec:appx for defining semantics} for an example of a finitary AF with high \compordinal $\alpha_S(a)$.


%
%
%
%

\begin{remark}
Although working with arbitrarily large countable ordinals may seem challenging, we stress that this is not fundamentally worse than the situation for the grounded extension, which can likewise require a transfinite number of steps---potentially up to any countable ordinal---to stabilize
 (see Andrews-San Mauro \cite{TARK}). 
  However, an important distinction arises in the finitary case: 
 for finitary AFs,  the grounded extension is found after $\omega$ steps \cite[Theorem 47]{dung1995acceptability}, whereas the ordinal values $\alpha_S(x)$ associated with the transfinite extensions of the SCC-recursive semantics  can be arbitrarily large \emph{also} in the finitary setting.
\end{remark}

We conclude this subsection with showing that, whenever $\cft$ and $\stgt$ are well-defined, they  coincide with their transfinite counterparts (proof in Appendix \ref{sec:appx for defining semantics}):

\begin{restatable}{theorem}{cftAndtfcftagree}
    \label{thm:when defined cf2 and icft agree}
    Let $\F$ be an AF and $S\subseteq A_\F$. Suppose that whether or not $S$ is a \cft-extension is well-defined. Then $S$ is a \cft-extension if and only if $S$ is a $\icft$-extension.
Similarly, suppose that whether or not $S$ is a \stgt-extension is well-defined. Then $S$ is an \stgt-extension if and only if $S$ is a \istgt-extension.
\end{restatable}



\subsection{SCC-prioritization instead of recursion}

We now propose a second approach that retains some of the benefits of SCC-recursiveness,
while avoiding the need for transfinite recursion. This method is characterized by evaluating strongly connected components as they appear in the original framework $\F$, without further subdividing a component $X$ in response to attacks from arguments in $S$ that lie  outside of $X$.

\begin{definition}\label{def:1.5 semantics}
    Let $\F=(A_\F, R_\F)$ and let  $S\in \cf(\F)$. Then, 
    \begin{itemize}
        \item $S\in \cfo(\F)$ iff, for each $X\in \SCC(\F)$, $S\cap X$ is a naive extension in $\F\restriction_{X\smallsetminus D_S(X)}$;
        \item $S\in \stgo(\F)$ iff, for each  $X\in \SCC(\F)$, $S\cap X$ is a stage extension in $\F\restriction_{X\smallsetminus D_S(X)}$.
    \end{itemize}
\end{definition}


\begin{remark}
    Note that, given any semantics $\sigma$, one  could similarly define the semantics $\sigma$\textit{1.5}, by simply saying that $S\in \sigma\textit{1.5} (\F)$ iff, for each $X\in \SCC(\F)$, $S\cap X$ is a $\sigma$-extension in $\F\restriction_{X\smallsetminus D_S(X)}$. 
\end{remark}

As a first test of our proposed semantics, we briefly revisit Examples 4-6 from Baroni and Giacomin
\cite{BaroniGiacomin03}, where the authors highlight certain undesirable behaviors exhibited by the preferred and grounded semantics on argumentation frameworks containing odd-length cycles. In particular, they observe that the preferred semantics treats odd and even cycles differently---an asymmetry which is regarded as problematic. For each of these examples, it is straightforward to  check that $\cft(\F)=\cfo(\F)=\stgo(\F)$, indicating that our semantics avoid some of the undesirable behaviors of the preferred and grounded semantics.

Example 7 from the same paper---reproduced here as Figure \ref{baronigiacomin-example}---reveals a distinction between $\cft$ and $\cfo$; however, in this example, $\stgo$ continues to align with both $\cft$ and $\stgt$.

\begin{figure}\centering
\scalebox{0.8}{
\begin{tikzpicture}[->=Stealth, state/.style={circle, draw, fill=mygray, minimum size=1cm}]

    \node[state, 
    ] (a) {$a$};
    \node[state, 
    ] (b0) [right=of a] {$b_0$};
   \node[draw=none] (nu) [right= of b0] {};
    \node[state, 
    ] (b2) [right=of nu] {$b_2$};
    \node[state, 
    ] (b1) [above=of nu] {$b_1$};
    \node[state, 
    ] (b3) [below=of nu] {$b_3$};

 \draw[->] (a) edge (b0);
 \draw[->] (b0) edge (b1);
 \draw[->] (b1) edge (b2);
 \draw[->] (b2) edge (b3);
 \draw[->] (b3) edge (b0);





\end{tikzpicture}
}

\caption{An example of a finite AF $\F$, where the semantics $\cft$ and $\cfo$ differ. Specifically, $\cft(\F)=\{\{a, b_1, b_3\}\}$, while $\cfo(\F)=\{\{a, b_1, b_3 \},\{a,b_2\}\}$; however, note that $\stgo(\F)=\stgt(\F)=\cft(\F)$.}
\label{baronigiacomin-example}

\end{figure}
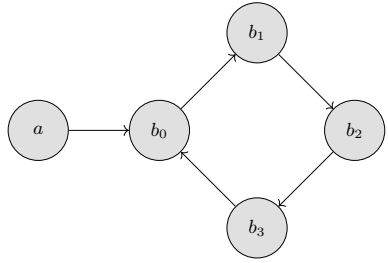

Similarly, consider a finite variant of the argumentation framework depicted in Figure \ref{BS-example}; for instance, $\F\restriction_{\{a_i,b_i \, : \, i\leq 4\}}$. In this case, the set $\{b_1,a_2,a_4\}$ is a $\cfo$ extension, but not a $\cft$ extension. As before, the unique $\cft$ extension, and likewise the unique $\stgo$ extension, is $\{b_i : i\leq 4\}$.

Thus, considering motivating examples, $\cfo$ and $\stgo$ appear to be promising candidates for well-defined semantics that preserve some of the key advantages of SCC-recursiveness, while avoiding the complexity of deep transfinite recursion. In the next section, we critically examine the trade-offs involved in adopting these semantics---alongside with  \icft and \istgt---for reasoning in the infinite setting.


\section{Properties of these semantics}\label{sec:properties}
To clarify the relative advantages of the semantics under consideration, we adopt the following naive-based evaluation criteria introduced by Baroni and Giacomin \cite{BaroniGiacomin}: 

\begin{definition}
A semantics $\sigma$ satisfies:
\begin{itemize}
    \item \textbf{I-maximality criterion} if, for each AF $\F$ and for each $S_1, S_2 \in \sigma(\F)$, $S_1 \subseteq S_2$ implies $S_1 = S_2$;
    \item \textbf{Reinstatement criterion} if, for each AF  $\F$ and for each $S \in \sigma(\F)$, if $S$ defends some argument $a$, then $a \in S$;
    \item \textbf{Weak reinstatement criterion} if, for each AF $\F$ and for each $S \in \sigma(\F)$, $S$ contains the grounded extension of $\F$;
    \item \textbf{CF-reinstatement criterion} if, for each AF $\F$ and for each $S \in \sigma(\F)$, whenever $S$ defends $a$ and $S\cup \{a\}$ is conflict-free, then $a \in S$;    
    \item \textbf{Directionality criterion} if, for each AF $\F$ and $U\subseteq A_\F$ which is not attacked from outside $U$, $\sigma(\F\restriction_U) = \{ S \cap U : S \in \sigma(\F) \}$.
\end{itemize}
\end{definition}


Baroni and Giacomin \cite[Proposition 63]{BaroniGiacomin} show that in the finite setting, $\cft$ satisfies the properties of elementary and weak skepticism adequacy, which we now define. We will examine below which variants of $\cft$ or $\stgt$ satisfy these conditions in the infinite setting.

\begin{definition}
Let $\tau_1$ and $\tau_2$ be two sets of extensions of an AF $\F$. Then:
\begin{itemize}
    \item The \textbf{elementary skepticism relation}, denoted $\tau_1 \preceq^E_\cap \tau_2$, holds iff
    \[
    \bigcap_{S_1 \in \tau_1} S_1 \subseteq \bigcap_{S_2 \in \tau_2} S_2;
    \]
    \item The \textbf{weak skepticism relation}, denoted $\tau_1\preceq^E_W \tau_2$, holds iff
    $(\forall S_2 \in \tau_2 \; \exists S_1 \in \tau_1)(S_1 \subseteq S_2)$.
\end{itemize}
\end{definition}

\begin{definition}
 For $\F=(A_\F,R_\F)$, let $\text{conf}(\F)$ be the set of conflicting pairs in $\F$: i.e., $\text{conf}(\F)=\{(x,y)\mid (x,y)\in R\text{ or }(y,x)\in R\}$.  Next, let $\sigma$ be a semantics:
\begin{itemize}
    \item For each of the skepticism relations $\preceq$, we say that $\sigma$ is \textbf{$\preceq$-skepticism adequate}, if $\F=(A_{F},R_\F)$ and $\G=(A_\G,R_\G)$ with $R_\F\supseteq R_\G$ and $\text{conf}(F)=\text{conf}(G)$ implies that $\sigma(\F)\preceq \sigma(\G)$.
\end{itemize}
\end{definition}

Table~\ref{table:properties} summarizes which of the considered semantics satisfy each of the evaluation criteria discussed above. The proofs of these results are provided in Appendix~\ref{sec:appx properties}.

\begin{table}[htbp]
\centering
\begin{tabular}{|l|l|l|l|l|l|l|l|l|}
\hline
 & \na & \cft & \stg & \stgt & \icft & \cfo & \istgt & \stgo \\ \hline
Well-defined & Y & N & Y & N & Y & Y & Y & Y \\ \hline
I-maximality & Y & Y & Y & Y & Y & Y & Y & Y \\ \hline
Reinstatement & N & N & N & N & N & N & N & N \\ \hline
Weak reinstatement & N & Y* & N & Y* & Y & N & Y & N \\ \hline
CF-reinstatement & Y & Y & Y & Y & Y & Y & Y & Y \\ \hline
Directionality & N & N & N & N & N & N & N & N \\ \hline
$\preceq^E_\cap$-sk.\ ad.\ & Y & N & N & N & Y & Y & N & N \\ \hline
$\preceq^E_w$-sk.\ ad.\ & Y & N & N & N & Y & Y & N & N \\ \hline
\end{tabular}
\caption{Summary of which semantics satisfy each of the evaluation criteria. Entries marked with Y* indicate cases where the criterion may be satisfied vacuously due to the semantics being undefined in some instances. For example, while every  $\cft$-extension  contains the grounded extension, it is possible for no $\cft$-extensions to exist,  meaning that arguments in the grounded extension could fail to be credulously $\cft$ accepted.
Results for the naive and stage semantics are from \cite{GagglDissertation}. 
}
\label{table:properties}
\end{table}

We regard directionality as a core desiratum for both the $\cft$ and $\stgt$ semantics; however, every version of these semantics fails to satisfy directionality in general. 

\begin{example}\label{ex:failureDirectionality}
    Let $\Sigma$ be the following collection of semantics: $\cft$, $\icft$, $\cfo$, $\stgt$, $\istgt$, and $\stgt$. Consider the AF $\F$ with argument set $A_\F:=\{a_i : i\in \nat\}$, and attacks defined by $a_i\att a_j$ if and only if $i>j$. It is straightforward to verify that, for $\sigma\in \Sigma$, $\F$ has no $\sigma$-extension. 

    Now, let $\G$ be the framework obtained by extending $\F$ with two additional arguments $x$ and $y$, which attack each other, and where $x\att a_i$, for all $i\in \nat$. In $\G$, the only $\sigma$-extension for any $\sigma\in \Sigma$ is $\{x\}$. However, in the subframework $\G\restriction_{\{x,y\}}$, we have  $\sigma(\G\restriction_{\{x,y\}})=\{\{x\},\{y\}\}$. Since $\{x,y\}$ is not attacked by any other argument in $\F$, we deduce that the directionality criterion fails for all the semantics in $\Sigma$.
\end{example}

In the next section, we will observe a different situation when restricting attention to finitary AFs.

\section{The finitary case}\label{sec:finitary case}
We view the failure of directionality shown in Example \ref{ex:failureDirectionality} as a significant concern for the  robustness of the semantics in question. 
The failure of directionality in $\G$ is an immediate consequence of the failure of existence of a $\sigma$-extension in $\F$, where each argument has infinitely many attackers. This observation naturally leads us to revisit a conjecture posed by Baumann and Spanring \cite[Conjecture 1]{baumann2015infinite} which asserts, when $\F$ is finitary, the \cf2 and \stg2 semantics should always admit at least one extension. This conjecture predates the authors' later discovery that these semantics are not well-defined in the general infinite setting. Now that we have introduced two distinct reformulations of \cf2 and \stg2 capable of handling infinite frameworks, we are in a position to re-examine this conjecture through the lens of our proposed semantics.

\begin{definition} For a semantics $\sigma$,
\begin{itemize}
    \item $\sigma$ satisfies \textbf{finitary existence} if, whenever $\F$ is a finitary AF, then there exists some $S\in \sigma(\F)$;
    \item $\sigma$ satisfies \textbf{finitary directionality} if, whenever $\F$ is finitary and $U\subseteq A_\F$ is not attacked from outside $U$, $\sigma(\F\restriction_U) = \{ S \cap U : S \in \sigma(\F) \}$.
\end{itemize}
\end{definition}

\begin{table}[htbp]
\centering
\begin{tabular}{|l|l|l|l|l|l|l|l|l|}
\hline
& \na & \cf2 & \stg & \stgt & \icft & \cfo & \istgt & \stgo \\ \hline
Finitary existence & Y & N & Y & N & ? & Y & N & Y \\ \hline
Finitary directionality & N & N & N & N & ? & Y & N & Y \\ \hline
\end{tabular}
\caption{Summary of results on finitary existence and directionality for the semantics under consideration.}
\label{table:finitary-properties}
\end{table}

Table \ref{table:finitary-properties} summarizes our results on this topic.
We begin by examining the property of finitary existence. Example~\ref{BS-example} demonstrates that there exists a finitary argumentation framework that admits neither a  \cft nor a \stgt-extension, hence the semantics \cft and \stgt do not satisy finitary existence. This is due to the fact that these semantics are  not well-defineded in that setting; in fact, the AF from Example \ref{BS-example} does have an \icft and \istgt extension---specifically, $\{b_i : i\in \nat\}$. Nonetheless, we now show that not every finitary AF has an \istgt extension.

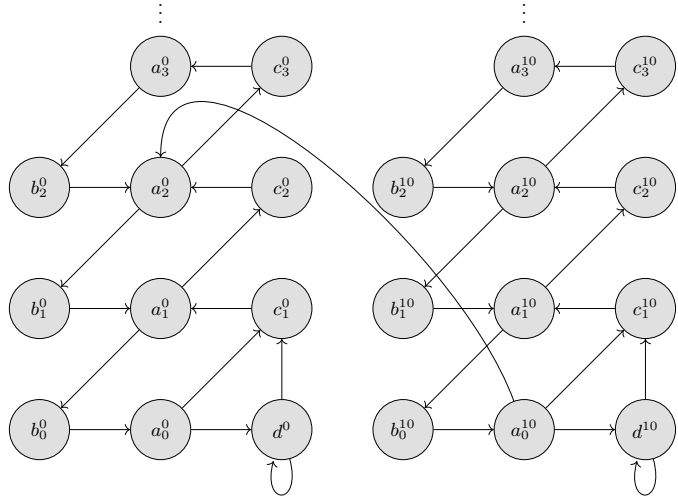
\begin{figure}\centering
\scalebox{0.8}{
\begin{tikzpicture}[->=Stealth, state/.style={circle, draw, fill=mygray, minimum size=1cm}]

    \node[state, 
    ] (b0) {$b^0_0$};
    \node[state, 
    ] (a0) [right=of b0] {$a^0_0$};
    \node[state, 
    ] (b1) [above=of b0] {$b^0_1$};
    \node[state, 
    ] (b2) [above=of b1] {$b^0_2$};
    \node[state, 
    ] (a1) [above=of a0] {$a^0_1$};
    \node[state, 
    ] (a2) [above=of a1] {$a^0_2$};
     \node[state, 
    ] (a3) [above=of a2] {$a^0_3$};
    \node[state, 
    ] (c1) [right=of a1] {$c^0_1$};
    \node[state, 
    ] (c2) [right=of a2] {$c^0_2$};
    \node[state, 
    ] (c3) [right=of a3] {$c^0_3$};

    \node[draw=none] (ellip)[above=1mm of a3] {$\vdots$};

     \node[state, 
    ] (d) [right=of a0] {$d^0$};

    \node[state, 
    ] (1b0) [right=of d] {$b^{10}_0$};
    \node[state, 
    ] (1a0) [right=of 1b0] {$a^{10}_0$};
    \node[state, 
    ] (1b1) [above=of 1b0] {$b^{10}_1$};
    \node[state, 
    ] (1b2) [above=of 1b1] {$b^{10}_2$};
    \node[state, 
    ] (1a1) [above=of 1a0] {$a^{10}_1$};
    \node[state, 
    ] (1a2) [above=of 1a1] {$a^{10}_2$};
     \node[state, 
    ] (1a3) [above=of 1a2] {$a^{10}_3$};
    \node[state, 
    ] (1c1) [right=of 1a1] {$c^{10}_1$};
    \node[state, 
    ] (1c2) [right=of 1a2] {$c^{10}_2$};
    \node[state, 
    ] (1c3) [right=of 1a3] {$c^{10}_3$};

    \node[draw=none] (1ellip)[above=1mm of 1a3] {$\vdots$};

     \node[state, 
    ] (1d) [right=of 1a0] {$d^{10}$};

  \draw[->] (b0) edge (a0);
  \draw[->] (b1) edge (a1);
  \draw[->] (b2) edge (a2);

  \draw[->] (a0) edge (c1);
  \draw[->] (a1) edge (c2);
  \draw[->] (a2) edge (c3);

    \draw[->] (c1) edge (a1);
  \draw[->] (c2) edge (a2);
  \draw[->] (c3) edge (a3);

   \draw[->] (a1) edge (b0);
  \draw[->] (a2) edge (b1);
  \draw[->] (a3) edge (b2);

    \draw[->] (a0) edge (d);

    \draw[->] (d) edge (c1);

    \draw[->] (d) edge[loop below] (d);

     \draw[->] (1b0) edge (1a0);
  \draw[->] (1b1) edge (1a1);
  \draw[->] (1b2) edge (1a2);

  \draw[->] (1a0) edge (1c1);
  \draw[->] (1a1) edge (1c2);
  \draw[->] (1a2) edge (1c3);

    \draw[->] (1c1) edge (1a1);
  \draw[->] (1c2) edge (1a2);
  \draw[->] (1c3) edge (1a3);

   \draw[->] (1a1) edge (1b0);
  \draw[->] (1a2) edge (1b1);
  \draw[->] (1a3) edge (1b2);

    \draw[->] (1a0) edge (1d);

    \draw[->] (1d) edge (1c1);

    \draw[->] (1d) edge[loop below] (1d);

     \draw[->] (1a0) edge[out=110, in=90] (a2);





\end{tikzpicture}
}

\caption{Fragment of a finitary AF encoding  a tree $T\supset \{0,10\}$ in its strongly connected components. We assume that the string $10$ is the second string in the enumeration of all strings that are longer than the string $0$ and incomparable with it.}
\label{fig: no tf-stg2 extension}

\end{figure}

\begin{theorem}
    There is a finitary AF with no \istgt extension.
\end{theorem}
\begin{proof}
    Let $T\subseteq \nat^{<\nat}$ be a tree of infinite height with no path. We  construct a finitary AF consisting of infinitely  many strongly connected components, indexed by $\sigma\in T$. For each $\sigma$, the component $X_\sigma$ contains the arguments: $\{a_i^\sigma : i\in \nat\}\cup \{b_i^\sigma : i\in \nat\}\cup \{c_i^\sigma : i\geq 1\}\cup \{d^\sigma\}$. The attack relations within $X_\sigma$ are defined as follows (see Figure \ref{fig: no tf-stg2 extension}):\\
    For $i\in\nat$,
    \begin{itemize}
        \item  $a^\sigma_i\att c_{i+1}^\sigma\att a_{i+1}^\sigma \att b_{i}^\sigma \att a^\sigma_i$;
        \item $a^\sigma_0 \att d^\sigma \att c^\sigma_1$;
        \item  $d^\sigma\att d^\sigma$.
    \end{itemize}

We define the full AF $\F$ with argument set $\A_\F=\bigcup_{\sigma\in T} X_\sigma$. Now, for each $\sigma\in T$, fix an enumeration $(\tau_i)_{i\in\nat}$ of all strings $\tau\in T$ longer than $\sigma$ and so $\sigma\not\preceq \tau$. We then let $a_0^{\tau_i}\att a^\sigma_i$.

\smallskip

    We argue that $\F$ has no \istgt extension. Towards a contradiction, suppose that $S$ were such an \istgt extension.
    
    First, observe that a \stg-extension of a component $X_\sigma$ must include $a^\sigma_0$. Indeed, the set $Z=\{a^\sigma_i : i\in \nat\}$ forms a stage extension with $Z^\oplus=X_\sigma$, and the only way to include $d\in Z^{\oplus}$ is for a stage extension $Z$ is to include $a_0^\sigma$. 
    It follows that for every $\sigma$, $S$ must contain some $a_0^\tau$ where $\tau$ has length $\geq \vert \sigma\vert$. 

    Next, suppose that $\sigma$ is shorter than $\tau$ and both $a_0^\sigma$ and $a_0^\tau$ are in $S$. It must be that $\sigma\preceq \tau$, as otherwise $a_0^\tau$ would attack some $a_i^\sigma$ which would make $b^\sigma_{i-1}$ be in its own connected component. Since it is unattacked from $S$, because $a_i^\sigma\notin S$, this implies that $b^\sigma_{i-1}\in S$. This in turn attacks $a_{i-1}^\sigma$, and so on, until we see $a_0^\sigma$ is attacked by $b_0^\sigma\in S$, yielding a contradiction. Thus, there are arbitrarily long $\sigma$ so that $a_0^\sigma\in S$, and these are comparable. This implies that $T$ has a path, which is a contradiction.
\end{proof}

\begin{theorem}
    The semantics \cft, \stgt, and \istgt do not satisfy finitary directionality.
\end{theorem}
\begin{proof}
    We reason as in Example \ref{ex:failureDirectionality}. Let $\sigma\in \{\cft,\stgt, \istgt\}$; each of these semantics fail to satisfy finitary existence. Now, let $\F$ be a finitary AF which does not have a $\sigma$-extension. Let $G$ add two arguments $x,y$ which attack each other and such that, for all $z\in \F$, $x\att z$. Then, $G$ is finitary, yet fails directionality since $\{y\}$ is in $\sigma(G\restriction_{\{x,y\}})$ but the only $\sigma$-extension of $G$ is $\{x\}$.
\end{proof}

We now turn to proving that both \cfo and \stgo satisfy finitary existence and also finitary directionality. 
The core idea behind both arguments is the same: we exploit the compactness theorem for propositional logic to find an extension which is \cfo or \stgo.
%
The main challenge is that the space of choices used to construct a \cfo or \stgo extension may involve \emph{infinitely many choices}, which obstruct compactness. Indeed, if $Z$ is infinite, then $P_i\leftrightarrow (\bigwedge_{j\in Z} \neg P_j)$ is not a propositional formula and not subject to the compactness theorem. For instance, suppose that $\F$ consists of a single strongly connected component. To build a naive extension, for each argument $x$, we must either include $x$ in $S$, or some $y$ which is in conflict with $x$. However, there may be infinitely many $y$'s with $x\att y$.

A natural idea would be to restrict attention to sets $S$ such that every argument is either in  $S$ or is attacked by some element of 
$S$, thereby reducing the formulas describing the condition to a finite set of choices. Unfortunately, such an extension cannot be found in general---for example, no naive extension of a 3-cycle satisfies it. 

Our strategy, therefore, is to define a restricted class of potential 
 \cfo or \stgo extensions which is broad enough to guarantee that every finite AF admits an extension in this subclass,  yet narrow enough to ensure that the corresponding space of choices is compact.

\begin{lemma}\label{lem:orderingF}
    If $\F$ is a finitary AF, then there exists a well-ordering $<$ of $\A_\F$ such that each argument $a$ attacks only finitely many arguments $b$ with $b<a$. 
\end{lemma}
\begin{proof}
    Fix any enumeration $\A_\F=\{a_\gamma \mid \gamma<\kappa\}$ of order type an ordinal $\kappa$. We build another ordering $\A_\F=\{b_\gamma \mid \gamma<\kappa\}$ out of this one. For each ordinal $\beta$, we define $\gamma_\beta$ to be least so that $a_{\gamma_\beta}\notin \{b_\alpha \mid \alpha<\beta\}$.

    We let $b_0=a_0$. We then let $b_1,\ldots, b_k$ be all the attackers of $a_0$. Next, we let $b_{k+1}=a_{\gamma_{k+1}}$. We then let $b_{k+2},\ldots, b_\ell$ be all the attackers of any $b_m$ with $m\leq k+1$. Continuing as such, we see that if $x$ attacks $y$, then $x$ appears in our enumeration $b_\alpha$ at most finitely much after $y$.
\end{proof}

\begin{theorem}\label{thm:finitary existence cf2}
    If $\F$ is a finitary AF, there exists a $\cfo$ extension in $\F$. 
\end{theorem}
\begin{proof}
    By Lemma \ref{lem:orderingF}, we may assume that $A_\F$ is well-ordered in such a way that each element attacks at most finitely many before it. For any $a\in A_\F$, define $B(a)$ to be the finite set $\{x\in \SCC(a) : x<a \text{ and } a\att x\}$.
    We define an extension $S$ to be a \emph{greedy \cfo extension} if $S\in \cf(\F)$ and every element $a\in A_{\F}$ is either in $S$, or is attacked by an element of $S$, or there exists an argument in $S\cap B(a)$. 

    We now define a propositional theory associated with $\F$. For each $a\in A_\F$, introduce a propositional variable $P_a$. Define a theory $T_0$ consisting of the following formulas:
    For each $a\in A_\F$ with $a\natt a$, $$P_a\leftrightarrow (\bigwedge_{b\att a} \neg P_b \wedge \bigwedge_{b\in B(a)}\neg P_b).$$ Let $T_1$ be the theory that says for each $a\att b$, $\neg (P_a\wedge P_b)$. Finally, let $T=T_0\cup T_1$. 
    
    Note that $T$ is a propositional theory: in particular, since each argument in $F$ has only finitely many attackers, each conjuction in the above formulas is finite.

To apply the compactness theorem for propositional logic, we must show that every finite subset $T'$ of $T$ is consistent. Fix $T'$ a finite subset of $T$. Let $X$ be the finite set of arguments $a$ so that $P_a$  appears in $T'$.  We now argue that the framework $F\restriction_X$ has a greedy \cfo extension. This can be shown by starting with an initial strongly connected component $Y$ and using a greedy algorithm to find a naive extension:
\begin{itemize}
    \item Take the first element of $Y$ (using the order on $\F$) and put it into $S$ unless it attacks itself;
    \item Continuing taking successive elements of $Y$ and put them into $S$ unless they attack themselves or are in conflict with the elements previously put into $S$;
    \item Once finished with $Y$, we then proceed to another strongly connected component which is initial among the remaining strongly connected components.
\end{itemize}

As such, we can build a greedy \cfo extension of $F\vert_X$. We note that this satisfies the formulas of $T'$. By compactness, $T$ is consistent. 
    
    Let $\pi$ be a model of $T$. Then define $S$ by letting $a$ be in $S$ iff $\pi$ makes $P_a$ true. We observe that $S$ is a greedy \cfo extension of $F$, and is thus a \cfo extension of $F$.
    \end{proof}

A similar, albeit more involved, argument shows that the \stgo semantics satisfies finitary existence (see Appendix \ref{sec:finitary existence for stgo}).

\begin{restatable}{theorem}{stgFinitaryEx}\label{thm:finitary existence stg}
    If $\F$ is a finitary AF, then it has an \stgo-extension.
\end{restatable}

The following theorem does not follow immediately from finitary existence, but rather requires a new argument, in the \stgo case requiring another application of propositional compactness (see Appendix \ref{sec:finitary existence for stgo}). 

\begin{restatable}{theorem}{cfostgoFinitaryDirectionality}\label{thm:cfostgoFinitaryDirectionality}
    The semantics    \cfo, \stgo satisfy finitary directionality.
\end{restatable}






\section{Discussion}

The results presented in this paper contribute to the broader understanding of argumentation in infinite domains, a setting of increasing relevance for AI systems that must operate over unbounded or dynamically evolving data. Our two proposed approaches---the transfinite extension of SCC-recursive semantics and the introduction of the \cfo and \stgo semantics---highlight different design trade-offs when generalizing well-understood concepts from the finite case.

The transfinite extensions provide a principled way to salvage the recursive methodology of \cft and \stgt in the infinite setting, aligning closely with known alternative characterizations. However, our analysis reveals that these extensions can suffer from foundational issues, particularly around directionality and existence, in infinite frameworks. 

By contrast, the \cfo and \stgo semantics preserve the original SCC structure throughout the evaluation process. This approach avoids the need for transfinite recursion. Notably, in finitary AFs, \cfo and \stgo demonstrate superior adherence to the key evaluation criteria of directionality and existence, suggesting that they may serve as more robust alternatives for infinite argumentation reasoning tasks in practice.


Going forward, one promising direction is to study algorithmic properties and complexity results for the newly introduced semantics, particularly in the context of incremental or streaming argumentation systems. Additionally, we left open the question of finitary existence and directionality for the \icft semantics.

In sum, our work advances the theoretical foundations of infinite argumentation and provides practical tools for constructing semantics that remain meaningful in the presence of infinite interaction structures.

\begin{credits}
\subsubsection{\ackname} This work was supported by the the National Science Foundation under Grant DMS-2348792. San Mauro is a member of INDAM-GNSAGA.

\subsubsection{\discintname}
The authors have no competing interests to declare that are
relevant to the content of this article.
\end{credits}
%
%
%

 \bibliographystyle{splncs04}

%





\newpage

\appendix

\section{Detailed explanation of Remark \ref{rmk:Equivalence With Gaggl}}\label{sec:equivalence}

We recall the following definitions from Gaggl and Woltran \cite[Definitions 3.6, 3.7]{Gaggl-Woltran-cf2}:
\begin{definition}
    Let $\F$ be an AF, $a,b$ be arguments and $B\subseteq A_\F$. We say $b$ is \emph{reachable from $a$ modulo $B$}, written $a\Rightarrow^B_\F b$, if there is a path from $a$ to $b$ in $\F\vert_B$. 
\end{definition}

\begin{definition}
    For an AF $\F$ and $D, S\subseteq A_\F$: 
    $$\Delta_{\F,S}(D)=\{a\in A_\F \mid \exists b\in S (b\att a \wedge a\not\Rightarrow_\F^{A_\F\smallsetminus D} b)\}.$$
\end{definition}

The following definition was given for finite ordinals, which we expand to infinite ordinals as well (we note a small change in the definition at index 0, which only shifts the indices by 1):
\begin{definition}
    $\Delta^0_{\F,S}=\emptyset$.

    $\Delta^{\alpha+1}_{\F,S}=\Delta_{\F,S}(\Delta^\alpha_{\F,S})$.

    $\Delta^\lambda_{\F,S}=\bigcup_{\alpha<\lambda} \Delta^\alpha_{\F,S}$.
\end{definition}

We note that this is a situation of a familiar lattice-theoretic phenomenon:

\begin{definition}\label{def:lattice theoretic sequence}
    Let $L$ be a complete lattice and $f:L\to L$ be any function. For all ordinals $\alpha$ and $u\in L$, the \emph{$\alpha$-iteration of $f$} is inductively defined as follows:
    \begin{itemize}
        \item $f^0(u)=u$;
        \item If $\alpha=\beta+1$, then $f^\alpha(u)=f(f^\beta(u))$;
        \item If $\alpha$ is a limit ordinal, then $f^\alpha(u)=\bigvee_{\beta < \alpha} f^\beta(u)$.
    \end{itemize}
\end{definition}

\begin{theorem}[see, e.g., \protect{\cite[Corollary 3.7]{venema2008lectures}}]\label{thm:lattice-theory approximation}
    Let $L$ be a complete lattice with bottom element $0_L$, equipped with a monotone function $f:L\to L$. Then, there is an ordinal $\alpha$ so that the least fixed point of $f$ coincides with $f^\alpha(0_L)$.
\end{theorem}

Since $\Delta_{\F,S}$ is monotonic, it follows that $\Delta_{\F,S}$ has a least fixed point and it equals $\Delta^\alpha_{\F,S}$ for some $\alpha$ of cardinality no more than the cardinality of $\F$. We call this least fixed point $\Delta_{\F,S}$.

Note that every element of $\Delta_{\F,S}$ is attacked by some element of $S$.


\begin{lemma}\label{lem:connect C's to Deltas}
    For each $a$ and $\alpha$, $C_S^\alpha(a)=\{b \mid a\Rightarrow^{\A_\F\smallsetminus \Delta^\alpha_{\F,S}}_\F b \wedge b\Rightarrow^{\A_\F\smallsetminus \Delta^\alpha_{\F,S}}_\F a\}$
\end{lemma}
\begin{proof}
    For $\alpha=0$, every $C_S^0(a)$ is the strongly connected component of $a$ is exactly $\{b \mid a\Rightarrow^{\A_\F\smallsetminus \Delta^0_{\F,S}}_\F b \wedge b\Rightarrow^{\A_\F\smallsetminus \Delta^0_{\F,S}}_\F a\}$ since $\Delta^0_{\F,S}=\emptyset$.

    Recall $C_S^{\alpha+1}(a)$ is the strongly connected component of $a$ in $C^\alpha_S(a)\smallsetminus \{x \mid \exists y\in S\smallsetminus C_S^\alpha(a)(y\att x)\}$. This is exactly saying that $b\in C^{\alpha+1}_S(a)$ if and only if there is a path from $a$ to $b$ and a path from $b$ to $a$ each avoiding the elements attacked from outside $C_S^\alpha(a)$. By inductive hypothesis, $C_S^\alpha(a)$ is $\{b \mid a\Rightarrow^{\A_\F\smallsetminus \Delta^\alpha_{\F,S}}_\F b \wedge b\Rightarrow^{\A_\F\smallsetminus \Delta^\alpha_{\F,S}}_\F a\}$. Thus if $d\in S$ attacks $c$ and $d$ is outside $C_S^\alpha(a)$, then $c\not\Rightarrow^{\A_\F\smallsetminus \Delta^\alpha_{\F,S}} d$, i.e., $c\in \Delta^{\alpha+1}_{\F,S}$. Thus $C_S^{\alpha+1}(a)$ if and only if $a\Rightarrow^{\A_\F\smallsetminus \Delta^{\alpha+1}_{\F,S}}_\F b \wedge b\Rightarrow^{\A_\F\smallsetminus \Delta^{\alpha+1}_{\F,S}}_\F a$.
    


    Finally, $C^\lambda_S(a)$ is the connected component of $a$ in $\bigcap_{\alpha<\lambda}C_S^\alpha(a)$. By inductive hypothesis,  $\Delta_{\F,S}^\alpha(a)\cap C_S^\alpha(a)=\emptyset$. Thus for $b$ to be in $C^\lambda_S(a)$, there must be a path from $a$ to $b$ and back which avoids $\bigcup_{\alpha<\lambda} \Delta^\alpha_{\F,S}$, i.e, avoids $\Delta^{\lambda}_{\F,S}$. Thus $C_S^\lambda(a)\subseteq \{b \mid a\Rightarrow^{\A_\F\smallsetminus \Delta^\lambda_{\F,S}}_\F b \wedge b\Rightarrow^{\A_\F\smallsetminus \Delta^\lambda_{\F,S}}_\F a\}$.
    On the other hand, if a path from $b$ to $a$ is entirely in $\A_\F\smallsetminus \Delta^\lambda_{\F,S}$, then it is also contained in $\A_\F\smallsetminus \Delta^\alpha_{\F,S}$ for each $\alpha$. Thus it is entirely in $\bigcap_{\alpha<\lambda}C_S^\alpha(a)$. Thus $C_S^\lambda(a)\supseteq \{b \mid a\Rightarrow^{\A_\F\smallsetminus \Delta^\lambda_{\F,S}}_\F b \wedge b\Rightarrow^{\A_\F\smallsetminus \Delta^\lambda_{\F,S}}_\F a\}$.
\end{proof}

\begin{lemma}
       If $\beta > \alpha_S(a)$ for every $a\in \F$, then $\Delta_{\F,S}=\Delta_{\F,S}^\beta$. 
\end{lemma}
\begin{proof}
    It follows from Lemma \ref{lem:connect C's to Deltas} that whenever $a\in \Delta^{\alpha+1}_{\F,S}\smallsetminus \Delta^\alpha_{\F,S}$, then $a\in C_S^\alpha(a)$ and $a\notin C_S^{\alpha+1}(a)$. Thus $\alpha+1=\alpha_S(a)$, showing that $\beta >\alpha+1$. In particular, we cannot have any $a\in \Delta^{\beta+1}_{\F,S}\smallsetminus\Delta^\beta_{\F,S}$. That is, $\Delta^{\beta+1}_{\F,S}=\Delta^\beta_{\F,S}$ 
\end{proof}

Recall \cite[Definition 3.1]{Gaggl-Woltran-cf2} that for an AF $\F$, $[[\F]]$ is the separation of $\F$, defined by removing attacks between elements of different strongly connected components in $\F$. 

\begin{theorem}\label{thm:icft-theequivalence}
    An extension $S$ is in \icft$(\F)$ if and only if $S$ is conflict-free and $S\in \na([[\F\smallsetminus \Delta_{\F,S}]])$.
\end{theorem}
\begin{proof}
    Let $\beta$ be large enough that $\beta > \alpha_S(a)$ for each $a\in A_\F$. Then $S\in \icft$ if it is conflict-free and, for each non-empty $C_S^\beta(a)$, $S\cap C_S^\beta(a)$ is a naive extension in $C_S^\beta(a)$. By Lemma \ref{lem:connect C's to Deltas}, the sets $C_S^\beta(a)$ are exactly the connected components in $\F\smallsetminus \Delta^\beta_{\F,S}=\F\smallsetminus \Delta_{\F,S}$. Being a naive extension in each of these connected components is the same as being a naive extension in $[[\F\smallsetminus \Delta_{\F,S}]]$.
\end{proof}

\begin{theorem}\label{thm:istgt-theequivalence}
    An extension $S$ is in \istgt$(\F)$ if and only if $S$ is conflict-free and $S\in \stg([[\F\smallsetminus \Delta_{\F,S}]])$.
\end{theorem}
\begin{proof}
    Let $\beta$ be large enough that $\beta > \alpha_S(a)$ for each $a\in A_\F$. Then $S\in \istgt$ if it is conflict-free and, for each non-empty $C_S^\beta(a)$, $S\cap C_S^\beta(a)$ is a stage extension in $C_S^\beta(a)$. By Lemma \ref{lem:connect C's to Deltas}, the sets $C_S^\beta(a)$ are exactly the connected components in $\F\smallsetminus \Delta^\beta_{\F,S}=\F\smallsetminus \Delta_{\F,S}$. Being a stage extension in each of these connected components is the same as being a stage extension in $[[\F\smallsetminus \Delta_{\F,S}]]$.
\end{proof}

\section{Details of Section \ref{sec:defining semantics}}\label{sec:appx for defining semantics}

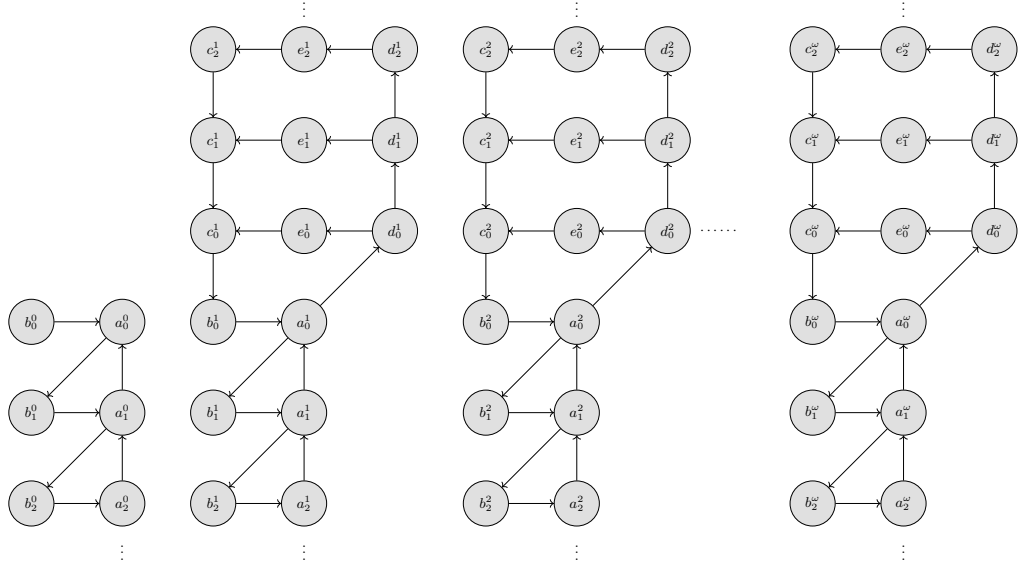
\begin{figure}
\scalebox{0.6}{
    \begin{tikzpicture}[->=Stealth, state/.style={circle, draw, fill=mygray, minimum size=1cm}]
    
    \node[state] (b00) {$b^0_0$};
    \node[state] (b01) [below=of b00] {$b^0_1$};
    \node[state] (b02) [below=of b01] {$b^0_2$};

    \node[state] (a00) [right=of b00] {$a^0_0$};
    \node[state] (a01) [below=of a00] {$a^0_1$};
    \node[state] (a02) [below=of a01] {$a^0_2$};

    \node[state] (b10) [right=of a00] {$b^1_0$};
    \node[state] (b11) [below=of b10] {$b^1_1$};
    \node[state] (b12) [below=of b11] {$b^1_2$};

    \node[state] (c10) [above=of b10] {$c^1_0$};
    \node[state] (c11) [above=of c10] {$c^1_1$};
    \node[state] (c12) [above=of c11] {$c^1_2$};

    \node[state] (a10) [right=of b10] {$a^1_0$};
    \node[state] (a11) [below=of a10] {$a^1_1$};
    \node[state] (a12) [below=of a11] {$a^1_2$};

    \node[state] (e10) [right=of c10] {$e^1_0$};
    \node[state] (e11) [above=of e10] {$e^1_1$};
    \node[state] (e12) [above=of e11] {$e^1_2$};

    \node[state] (d10) [right=of e10] {$d^1_0$};
    \node[state] (d11) [above=of d10] {$d^1_1$};
    \node[state] (d12) [above=of d11] {$d^1_2$};

    \node[state] (c20) [right=of d10] {$c^2_0$};
    \node[state] (c21) [above=of c20] {$c^2_1$};
    \node[state] (c22) [above=of c21] {$c^2_2$};

    \node[state] (e20) [right=of c20] {$e^2_0$};
    \node[state] (e21) [above=of e20] {$e^2_1$};
    \node[state] (e22) [above=of e21] {$e^2_2$};

    \node[state] (d20) [right=of e20] {$d^2_0$};
    \node[state] (d21) [above=of d20] {$d^2_1$};
    \node[state] (d22) [above=of d21] {$d^2_2$};
    \node[draw=none] (ellip)[right=1mm of d20] {$\cdots\cdots$};

    \node[state] (b20) [below=of c20] {$b^2_0$};
    \node[state] (b21) [below=of b20] {$b^2_1$};
    \node[state] (b22) [below=of b21] {$b^2_2$};

    \node[state] (a20) [right=of b20] {$a^2_0$};
    \node[state] (a21) [right=of b21] {$a^2_1$};
    \node[state] (a22) [right=of b22] {$a^2_2$};

    \node[state] (com0) [right=of ellip] {$c^\omega_0$};
    \node[state] (com1) [above=of com0] {$c^\omega_1$};
    \node[state] (com2) [above=of com1] {$c^\omega_2$};

    \node[state] (eom0) [right=of com0] {$e^\omega_0$};
    \node[state] (eom1) [right=of com1] {$e^\omega_1$};
    \node[state] (eom2) [right=of com2] {$e^\omega_2$};

    \node[state] (dom0) [right=of eom0] {$d^\omega_0$};
    \node[state] (dom1) [right=of eom1] {$d^\omega_1$};
    \node[state] (dom2) [right=of eom2] {$d^\omega_2$};

    \node[state] (bom0) [below=of com0] {$b^\omega_0$};
    \node[state] (bom1) [below=of bom0] {$b^\omega_1$};
    \node[state] (bom2) [below=of bom1] {$b^\omega_2$};

    \node[state] (aom0) [right=of bom0] {$a^\omega_0$};
    \node[state] (aom1) [below=of aom0] {$a^\omega_1$};
    \node[state] (aom2) [below=of aom1] {$a^\omega_2$};

    \draw[->] (b00) edge (a00);
    \draw[->] (b01) edge (a01);
    \draw[->] (b02) edge (a02);

    \draw[->] (b10) edge (a10);
    \draw[->] (b11) edge (a11);
    \draw[->] (b12) edge (a12);

    \draw[->] (b20) edge (a20);
    \draw[->] (b21) edge (a21);
    \draw[->] (b22) edge (a22);

    \draw[->] (bom0) edge (aom0);
    \draw[->] (bom1) edge (aom1);
    \draw[->] (bom2) edge (aom2);

    \draw[->] (e10) edge (c10);
    \draw[->] (e11) edge (c11);
    \draw[->] (e12) edge (c12);

    \draw[<-] (e10) edge (d10);
    \draw[<-] (e11) edge (d11);
    \draw[<-] (e12) edge (d12);

    \draw[->] (e20) edge (c20);
    \draw[->] (e21) edge (c21);
    \draw[->] (e22) edge (c22);

    \draw[<-] (e20) edge (d20);
    \draw[<-] (e21) edge (d21);
    \draw[<-] (e22) edge (d22);

    \draw[->] (eom0) edge (com0);
    \draw[->] (eom1) edge (com1);
    \draw[->] (eom2) edge (com2);

    \draw[<-] (eom0) edge (dom0);
    \draw[<-] (eom1) edge (dom1);
    \draw[<-] (eom2) edge (dom2);

    \draw[<-] (b01) edge (a00);
    \draw[<-] (b02) edge (a01);
    \draw[<-] (b11) edge (a10);
    \draw[<-] (b12) edge (a11);
    \draw[<-] (b21) edge (a20);
    \draw[<-] (b22) edge (a21);
    \draw[<-] (bom1) edge (aom0);
    \draw[<-] (bom2) edge (aom1);

    \draw[->] (a01) edge (a00);
    \draw[->] (a02) edge (a01);
    \draw[->] (a11) edge (a10);
    \draw[->] (a12) edge (a11);
    \draw[->] (a21) edge (a20);
    \draw[->] (a22) edge (a21);
    \draw[->] (aom1) edge (aom0);
    \draw[->] (aom2) edge (aom1);

       \draw[->] (aom0) edge (dom0);
      \draw[->] (dom0) edge (dom1);
      \draw[->] (dom1) edge (dom2);
        \draw[->] (com2) edge (com1);
      \draw[->] (com1) edge (com0);
      \draw[->] (com0) edge (bom0);

        \draw[->] (a10) edge (d10);
      \draw[->] (d10) edge (d11);
      \draw[->] (d11) edge (d12);
        \draw[->] (c12) edge (c11);
      \draw[->] (c11) edge (c10);
      \draw[->] (c10) edge (b10);

    \draw[->] (a20) edge (d20);
      \draw[->] (d20) edge (d21);
      \draw[->] (d21) edge (d22);
        \draw[->] (c22) edge (c21);
      \draw[->] (c21) edge (c20);
      \draw[->] (c20) edge (b20);
      \node[draw=none] (ellip2)[below=1mm of a02] {$\vdots$};
    \node[draw=none] (ellip3)[below=1mm of a12] {$\vdots$};
     \node[draw=none] (ellip4)[below=1mm of a22] {$\vdots$};
    \node[draw=none] (ellip5)[below=1mm of aom2] {$\vdots$};
    \node[draw=none] (ellip6)[above=1mm of e12] {$\vdots$};    
    \node[draw=none] (ellip7)[above=1mm of e22] {$\vdots$};    
    \node[draw=none] (ellip8)[above=1mm of eom2] {$\vdots$};
\end{tikzpicture}
}
\caption{A fragment of an AF---discussed in detail in Example \ref{example: high compordinal}.
For the sake of readability, some  attack relations have been omitted from the figure; see  Example \ref{example: high compordinal} for a full definition of the AF.
}
\label{fig: high compordinal}
\end{figure}


\begin{example}\label{example: high compordinal}
    Let $\F$ be the AF, (partially) depicted in Figure \ref{fig: high compordinal}, with 
    \[
    A_\F:= \{a^\alpha_i,b^\alpha_i,c^\beta_i,d^\beta_i, e^\beta_i : i \in \mathbb{N}, \alpha\leq \omega, 1\leq \beta\leq \omega\}
    \]
    and $R_\F$ consisting of the following attack relations: for $i \in \mathbb{N}$, $\alpha\leq \omega$, and $1\leq \beta\leq \omega$,
\begin{itemize}
    \item $b^\alpha_i \att a^\alpha_i, a^\alpha_i \att b^\alpha_{i+1},  a^\alpha_{i+1} \att a^\alpha_i$;
    \item $d^\beta_i \att e^\beta_i \att c^\beta_i, c^\beta_{i+1} \att c^\beta_i, d^\beta_{i} \att d^\beta_{i+1}, a^\beta_0 \att d^\beta_0, c^\beta_0 \att b^\beta_0$;
    \item $b^\alpha_i \att e^{\alpha+1}_i$, $e^{\alpha+1}_i\att a^\alpha_i$, $b^i_0\att e^\omega_i$, $e^\omega_i\att a^i_0$.
    \end{itemize}
    Let $S:=\{b^\alpha_i \mid i\in\mathbb{N}, \alpha\leq \omega\}\cup \{c^\alpha_{2i+1} \mid i\in\mathbb{N}, \alpha\leq \omega\}\cup \{d^\alpha_{2i} \mid i\in\mathbb{N}, 1\leq \alpha\leq \omega\}$.

    \smallskip
    
    We observe that $C_0^S(b_0^\omega)$ initially includes the entire framework except for $b_0^0$. After $\omega$  steps, $C_\omega^S(b_0^\omega)$ excludes all and only the arguments $\{b_i^0,a_i^0 \mid i\in \nat\}\cup \{c^1_i,d^1_i,e^1_i\mid i\in \nat\}$. More generally, after $\omega\cdot k$  steps, we have that 
    \[
    C_{\omega\cdot k}^S(b_0^\omega)=A_\F\smallsetminus\{b_i^j,a_i^j,c_i^l,d_i^l,e_i^l : i\in \omega, j\leq k,l\leq k+1\}.
    \]
 Finally, at stage $\omega^2$, we reach $C_{\omega^2}^S(b_0^\omega)=\{b_0^\omega\}$; hence, the \compordinal of the argument $b^\omega_0$ over $S$ is $\omega^2$.


The above construction can readily be adapted as a blueprint for building larger argumentation frameworks containing an element $x$ and a set $S$, so $\alpha_S(x)$ is an arbitrarily large countable ordinal. Essentially, we continue placing chunks as in the example so that $\alpha_S(b_0^\gamma)=\omega\cdot \gamma$. For limit ordinals $\lambda = \sup (Y)$ with $Y=\{\omega\cdot \gamma_i \mid i\in \nat\}$, we let $b_0^{\gamma_i}\att e_i^\lambda$ and $e_i^\lambda\att a_0^{\gamma_i}$. That said, some bookkeeping is necessary if one also wants to ensure that $\F$ remains finitary.

    \end{example}

\cftAndtfcftagree*
\begin{proof}
    Let $S$ to be a $\cft$ (respectively, $\stgt$) extension of $\F$. By Definitions \ref{def:cf2} and \ref{def:stg2}, the computation of 
$S$ proceeds via a recursive process that can be represented as a tree of recursive calls. Since $S$ is assumed to be a $\cft$ ($\stgt$) extension, this recursive tree  must be well-founded. Moreover, at each leaf of the tree, the process terminates with a strongly connected component $X$ such that $S\cap X$ is a naive (stage) extension of $\F\restriction X$. These terminal components are precisely the components $C_S^{\alpha_S(a)}(a)$, for $a\in A_\F$ so that $a\in C_S^{\alpha_S(a)}(a)$. This condition matches exactly the acceptance criteria for $S$ under the $\icft$ ($\istgt$) semantics; hence, $S$ is an  $\icft$ ($\stgt$) extension. 

Conversely, suppose that  $S$ is not a $\cft$ ($\stgt$) extension. Then,  the tree of recursive calls for computing $S$ must contain a leaf at which the process fails---specifically a terminal strongly connected component $X$ so that $S\cap X$ is not a naive (stage) extension in $X$. By construction, this component is equal to $C_S^{\alpha_S(a)}(a)$ for some $a$. Therefore,  $S$ fails to satisfy the acceptance condition for 
$\icft$ ($\istgt$) as well.
\end{proof}

\section{Details of Section \ref{sec:properties}}\label{sec:appx properties}

In this section, we give detailed proofs for the various properties established in Table \ref{table:properties}, or counterexamples where the properties do not hold. We repeat Table \ref{table:properties} here for convenience of reference.

\begin{table}[htbp]
\centering
\begin{tabular}{|l|l|l|l|l|l|l|l|l|}
\hline
 & na & cf2 & stg & stg2 & icft & cf1.5 & istg & stg1.5 \\ \hline
Well-Defined & Y & N & Y & N & Y & Y & Y & Y \\ \hline
I-maximality & Y & Y & Y & Y & Y & Y & Y & Y \\ \hline
Reinstatement & N & N & N & N & N & N & N & N \\ \hline
Weak Reinstatement & N & Y* & N & Y* & Y & N & Y & N \\ \hline
CF-Reinstatement & Y & Y & Y & Y & Y & Y & Y & Y \\ \hline
Directionality & N & N & N & N & N & N & N & N \\ \hline
$\preceq^E_\cap$-sk. ad. & Y & N & N & N & Y & Y & N & N \\ \hline
$\preceq^E_w$-sk. ad. & Y & N & N & N & Y & Y & N & N \\ \hline
\end{tabular}
\end{table}

\subsection*{Well-definedness} For well-definedness, we note that the non-well-definedness of the cf2 and stg2 semantics are established in \cite{BS-unrestricted}, as in Example \ref{BS-example}. The definitions are not recursive in \cfo and \stgo, so these are well-defined. Finally, the well-definedness for \icft and \istgt follows from Theorems \ref{thm:icft-theequivalence} and \ref{thm:istgt-theequivalence}. Since $\Delta$ is a monotone operator, Theorem \ref{thm:lattice-theory approximation} guarantees the existence of a least fixed-point and the conditions in Theorems \ref{thm:icft-theequivalence} and \ref{thm:istgt-theequivalence} are clearly well-defined. Alternatively, this follows since each $C_S^\alpha(a)$ must stabilize at some ordinal of size no more than the size of $A_\F$. This is because whenever $C^{\alpha+1}_S(a)\subsetneq C^\alpha_S(a)$, some element has left $C^{\alpha}_S(a)$. This can happen at most at $\vert C^0_S(a) \vert$ many ordinals. Thus $\alpha_S(a)$ is well-defined, showing that the notions of \icft and \istgt are well-defined.

\subsection*{I-maximality and CF-reinstatement}
Dvo\v r\'ak and Gaggl \cite[Proposition 3.12]{Dvorak-Gaggl-stg2} show that any semantics $\sigma$ so that $\sigma(\F)\subseteq \na(\F)$ satisfies I-maximality and CF-reinstatement. This shows that na and stg satisfy I-maximality and CF-reinstatement.

\begin{theorem}
    Let $\sigma$ be one of the semantics cf2, stg2, \icft, \cfo, \istgt, \stgo. Then $\sigma(\F)\subseteq \na(\F)$. Thus $\sigma$ satisfies I-maximality and CF-reinstatement.
\end{theorem}
\begin{proof}
    For cf2, stg2, \icft, \istgt, Theorem \ref{thm:icft-theequivalence} and Theorem \ref{thm:istgt-theequivalence}, in addition to the results of Gaggl and Woltran \cite[Theorem 3.11]{Gaggl-Woltran-cf2} and  Dvo\v r\'ak and Gaggl \cite[Proposition 3.2]{Dvorak-Gaggl-stg2}, show that any $\sigma$-extension is naive. In particular, every element of $\Delta_{\F,S}$ is attacked by an element of $S$ and every other element $a$ is part of its final strongly connected component $C_S^{\alpha_S(a)}(a)$ where $S$ is either naive or stage. Since stage extensions are all naive, $S$ is always naive in $C_S^{\alpha_S(a)}(a)$. 

    The same analysis works for \cfo and \stgo, except that we do not need to find a final strongly connected component. By definition, to be a \cfo extension, $S$ must be naive on $X\smallsetminus \{a\in X \mid \exists y\in (S\smallsetminus X) (y\att a\}$. Thus again every element not in $S$ is in conflict with $S$, showing that $S$ is naive. Since stage extensions are naive, the same argument works for \stgo.

    It follows from Dvo\v r\'ak and Gaggl \cite[Proposition 3.12]{Dvorak-Gaggl-stg2} that $\sigma$ satisfies I-maximality and CF-reinstatement.
\end{proof}

\subsection*{Reinstatement and Weak Reinstatement}

\begin{theorem}
    na, stg, cf2, \icft, \cfo, stg2, \istgt, \stgo do not satisfy reinstatement.
\end{theorem}
\begin{proof}
    Consider the AF of a 3-cycle: $a\att b\att c\att a$. In each of these semantics, $\{a\}$ is an accepted extension. But $\{a\}$ defends $c$.
\end{proof}

That the naive and stage semantics do not satisfy weak reinstatement is established in Gaggl \cite{GagglDissertation}.

\begin{theorem}\label{thm: icft istgt have weak reinstatement}
    \icft and \istgt satisfy weak reinstatement.
\end{theorem}
\begin{proof}
	We use the recursive definition of the grounded extension as in Theorem \ref{thm:lattice-theory approximation}. In particular, $\G_0=\emptyset$, $\G_{\alpha+1}=f_{\F}(G_{\alpha})$, and $\G_\lambda = \bigcup_{\alpha<\lambda} \G_\alpha$ for limit ordinals $\lambda$. Then $\G=\bigcup_{\alpha} \G_\alpha=\G_\beta$ for some countable ordinal $\beta$.
	
    We show by induction that
    $a\in \G_\alpha$ implies that $a\in S$ and $\C^\alpha_S(a)=\{a\}$. This is clearly true for $G_\emptyset$, which is empty by definition. Since $G_\lambda = \bigcup_{\alpha<\lambda} \G_\alpha$ for limit ordinals $\lambda$, the induction for limit ordinals is immediate. Suppose that $a\in \G_{\beta+1}$. Then we know that every element of $\G_\beta$ is in $S$. Further, $C^\beta_S(x)=\{x\}$ for each $x\in \G_\beta$. Then any element $b$ which attacks $a$ is defended by $\G_\beta$, thus is attacked from an element not in $C^\beta_S(b)$. Thus $C^{\beta+1}_S(a)=\{a\}$ since no attacker of $a$ can be in $C^{\beta+1}_S(a)$. Since every element which attacks $a$ is attacked from above by $S$, $C^{\gamma}_S(a)=\{a\}$ for all $\gamma>\beta$ showing that $\alpha_S(a)=\beta+1$. Finally, since $S\cap C^{\alpha_S(a)}_S(a)$ is a naive or stage extension in $C^{\alpha_S(a)}_S(a)$, it follows that $a\in S$. 
\end{proof}

\begin{corollary}
    The semantics cf2 and stg2 satisfy weak reinstatement.
\end{corollary}
\begin{proof}
    If $S$ is a cf2 or stg2 extension, then by Theorem \ref{thm:when defined cf2 and icft agree}, they are also an $\icft$ or $\istgt$ extension, thus must contain the grounded extension by Theorem \ref{thm: icft istgt have weak reinstatement}.
\end{proof}

\begin{figure}\centering
\scalebox{0.8}{
\begin{tikzpicture}[->=Stealth, state/.style={circle, draw, fill=mygray, minimum size=1cm}]

    \node[state, 
    ] (a) {$a$};
    \node[state, 
    ] (b0) [right=of a] {$b_0$};
   \node[draw=none] (nu) [right= of b0] {};
    \node[state, 
    ] (b2) [right=of nu] {$b_2$};
    \node[state, 
    ] (b1) [above=of nu] {$b_1$};
    \node[state, 
    ] (b3) [below=of nu] {$b_3$};

 \draw[->] (a) edge (b0);
 \draw[->] (b0) edge (b1);
 \draw[->] (b1) edge (b2);
 \draw[->] (b2) edge (b3);
 \draw[->] (b3) edge (b0);
 \draw[->] (b3) edge[loop below] (b3);

\end{tikzpicture}
}

\caption{An example of a finite AF $\F$, where the semantics $\cfo$ and $\stgo$ fail weak reinstatement.}
\label{fig:cfo stgo fail weak reinstatement}
\end{figure}
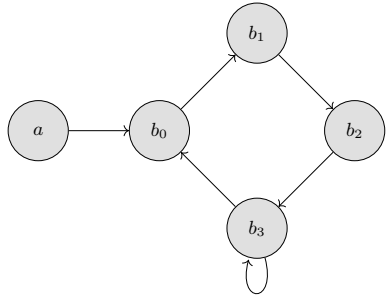

\begin{theorem}
    \cfo and \stgo do not satisfy weak reinstatement.
\end{theorem}
\begin{proof}
    Consider the AF in Figure \ref{fig:cfo stgo fail weak reinstatement}. Then $\{a,b_2\}$ is a \cfo and \stgo extension, yet $b_1\in \G$.
\end{proof}

\subsection*{Skepticism Adequacy}

That the naive semantics satisfies and the stage semantics does not satisfy $\preceq^E_\cap$- and $\preceq^E_W$-skepticism adequacy are established in Gaggl \cite{GagglDissertation}. 

We note that the proof in \cite{BaroniGiacomin} that cf2 satisfies $\preceq^E_W$-skepticism adequacy assumes a finite AF, since it only considers expanding the attack relation by one extra attack at a time. In fact, \cite[Proposition 2]{BaroniGiacomin} shows the following: 

\begin{proposition}[{\cite[Proposition 2]{BaroniGiacomin}}]\label{prop:BG-skep-ad}
    Suppose that $F=(A,R)$ and $G=(A,R')$ are \emph{finite AFs} with $R\subseteq R'$ where $\text{conf}(F)=\text{conf}(G)$. Let $S\in \cft(\F)$. Then $S\in \cft(\G)$.
\end{proposition}

Yet the following example shows Proposition \ref{prop:BG-skep-ad}, as well as both $\preceq^E_\cap$- and $\preceq^E_W$-skepticism adequacy fail in the infinite setting.

\begin{theorem}
    Both $\preceq^E_\cap$ and $\preceq^E_W$-skepticism adequacy fail for cf2.
\end{theorem}
\begin{proof}
    Let $G$ be Example \ref{BS-example} with one added argument $c$ so that $c\att a_i$ for all $i$. Note that $\{b_i \mid i\in \nat\}\cup \{c\}$ is a cf2 extension. Let $G'$ be as in $G$ except that we add the attacks $a_i\att c$. Note that just as in Example \ref{BS-example}, $G'$ does not have any cf2 extensions. Thus both $\preceq^E_\cap$- and $\preceq^E_W$-skepticism adequacy fail for cf2 in this example. This example can also be made finitary by adding, instead of one $c$, an infinite set of $c_i$ so that $c_i\att a_i$ in $G$ and also $a_i\att c_i$ in $G'$.
\end{proof}

\begin{theorem}\label{thm:icft skep ad}
    Suppose that $F=(A,R)$ and $G=(A,R')$ with $R\subseteq R'$ where $\text{conf}(F)=\text{conf}(G)$. Let $S\in \icft(\F)$. Then $S\in \icft(\G)$.

    In particular, \icft satisfies $\preceq^E_\cap$- and $\preceq^E_W$-skepticism adequacy. 
\end{theorem}
\begin{proof}
    Observe that $\Delta_{G,S}^\alpha\subseteq \Delta_{F,S}^\alpha$ for all $\alpha$ since reachability in $F$ implies reachability in $G$.
    It follows that if $a$ and $b$ are in the same strongly connected component in $[[F-\Delta_{F,S}]]$ then they are in the same strongly connected component in $[[G-\Delta_{G,S}]]$.
    We must show that $S$ is naive in $[[G-\Delta_{G,S}]]$. We consider a single strongly connected component $X$ in $G-\Delta_{G,S}$. 
    We can partition $X$ into two pieces: The elements in $\Delta_{F,S}$ and the remaining elements. These remaining elements are a union of strongly connected components in $[[F-\Delta_{F,S}]]$. We know that $S$ is naive on each strongly connected component in $[[F-\Delta_{F,S}]]$. For each element $x\in X$, either $x\in \Delta_{F,S}$ in which case it is attacked by an element of $S$ or it is either a self-attacker or is in conflict with some element in $S$ in its strongly connected component in $[[F-\Delta_{F,S}]]$. Thus $S$ is naive on $X$, showing that $S$ is in $\icft(\G)$.
\end{proof}

\begin{theorem}
    Suppose that $F=(A,R)$ and $G=(A,R')$ with $R\subseteq R'$ where $\text{conf}(F)=\text{conf}(G)$. Let $S\in \cfo(F)$. Then $S\in \cfo(G)$.

    In particular, \cfo satisfies $\preceq^E_\cap$- and $\preceq^E_W$-skepticism adequacy. 
\end{theorem}
\begin{proof}
    Let $X$ be a strongly connected component in $G$. Then $X$ is a union of strongly connected components in $F$. Let $\Theta$ be the set of strongly connected components in $F$ so that $X=\bigcup_{Y\in \Theta} Y$. We need to check that $S$ is a naive extension of $X\smallsetminus \{a\in X \mid \exists y\in S\smallsetminus X (y\att a)\}$.  For each $Z\in \Theta$, we know that $S$ is a naive extension of $Z\smallsetminus \{a\in Z \mid \exists y\in S\smallsetminus Z (y\att a)\}$. Thus, for each element $x\in X$, either $x$ is attacked by an element of $S$ or it is in one of the sets $Z\smallsetminus \{a\in Z \mid \exists y\in S\smallsetminus Z (y\att a)\}$ for $Z\in \Theta$. In this latter case, it is either in $S$ or it is a self-attacker or it is in conflict with an element of $S\cap Z\subseteq S\cap X$. It follows that $S$ is naive on $X\smallsetminus \{a\in X \mid \exists y\in S\smallsetminus X (y\att a)\}$.
\end{proof}

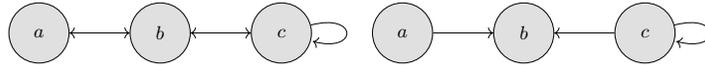
\begin{figure}\centering
\scalebox{0.8}{
\begin{tikzpicture}[->=Stealth, state/.style={circle, draw, fill=mygray, minimum size=1cm}]

    \node[state, 
    ] (a) {$a$};
    \node[state, 
    ] (b) [right=of a] {$b$};
    \node[state,] (c) [right= of b] {$c$};
    \node[state, 
    ] (d) [right=of c] {$a$};
    \node[state, 
    ] (e) [right=of d] {$b$};
    \node[state,] (f) [right= of e] {$c$};

 \draw[<->] (a) edge (b);
 \draw[<->] (b) edge (c);
 \draw[->] (c) edge[loop right] (c);
  \draw[->] (d) edge (e);
 \draw[->] (f) edge (e);
 \draw[->] (f) edge[loop right] (f);

\end{tikzpicture}
}

\caption{Two AFs $F$ and $G$ showing that \stg, \istgt, \stgo, each fail to satisfy $\preceq^E_\cap$- and $\preceq^E_W$-skepticism adequacy.}
\label{fig:fail skepticism adequacy}
\end{figure}

\begin{theorem}
    \stg, \istgt, \stgo, each fail to satisfy $\preceq^E_\cap$- and $\preceq^E_W$-skepticism adequacy.
\end{theorem}
\begin{proof}
    As in Proposition 8 of Gaggl \cite{GagglDissertation}: Let $F$ be the AF depicted on the left in Figure \ref{fig:fail skepticism adequacy}, and let $G$ be the AF depicted on the right in Figure \ref{fig:fail skepticism adequacy}. Then $\stg(F)=stg2(F)=\istgt(F)=\stgo(F)=\{\{b\}\}$ and $\stg(G)=stg2(G)=\istgt(G)=\stgo(G)=\{\{a\}\}$ showing that neither skepticism adequacy condition holds for these semantics.
\end{proof}

\section{Details for section \ref{sec:finitary case}}
\label{sec:finitary existence for stgo}

\begin{definition}
	We extend the $<$ order from Lemma \ref{lem:orderingF} to subsets of $\A_\F$ lexicographically:
	Let $X\neq Y$ be subsets of $\A_\F$. We say that $X$ is lexigraphically larger than $Y$ if $z\in X$ where $z$ is the least element of $\A_\F$ so $z\in X\not\leftrightarrow z\in Y$. Note that this is well-defined since $<$ is a well-order, and it linearly orders all subsets of $\A_\F$.
\end{definition}

\stgFinitaryEx*

\begin{proof}
    Let $<$ be an order on $\F$ as constructed in Lemma \ref{lem:orderingF}. Note that if $a\att b$, then $a$ appears at most finitely after $b$ in $<$:  i.e., the set $\{x\in \A_\F : a<x<b \}$ is finite. It follows that if there is any path from $a$ to $b$, then the interval $[a,b]$ is finite. By the fact that $<$ is a well-ordering, there is a least element of any SCC.  Thus for any $a$, the set of elements in $\SCC(a)$ to the left of $a$ is always finite.

    We form a propositional theory in predicates $P_a, Q_a$ for $a\in F$. We regard $P_a$ as saying that $a$ is in the extension $S$ which we are seeking. We regard $Q_a$ as saying that $a$ is in $S^\oplus$ for the extension $S$ which we are seeking.

    Let $T_0$ be the theory which says $\neg (P_a\wedge P_b)$ for each pair $(a,b)$ so that $a\att b$. Let $T_1$ be the formulas which say for each $a$: $Q_a\leftrightarrow(P_a \vee \bigvee_{b\att a}P_b)$. 
    Let $a$ be any element of $F$. Let $Z(a)$ be the finite set of elements in $\SCC(a)$ which are $\leq a$. Let $X$ be the finite set of elements in $(Z(a)\cup Z(a)^-)^-\smallsetminus \SCC(a)$. For every subset $Y\subseteq X$, let $Z^a_Y$ be the lexicographically-largest subset of $Z(a)$ which is contained in $(Y\cup S)^\oplus$ for any conflict-free set $Y\cup S$, where $S\subseteq (Z(a)\cup Z(a)^-)\cap \SCC(a)$. We let the formula $\psi_{a,Y}$ be the formula $(\bigwedge_{b\in Y} P_b\wedge \bigwedge_{b\notin Y} \neg P_b)\rightarrow \bigwedge_{c\in Z^a_Y} Q_c$. Finally, we let $T=T_0\cup T_1\cup \{\psi_{a,Y}\mid a\in \A_\F, Y\subseteq (Z(a)\cup Z(a)^-)^-\smallsetminus \SCC(a))\}$.

    In words, $T$ says that $S$ is conflict-free and that on each strongly connected component, $S^\oplus$ is lexicographically the largest possible. 
    
    We now check that $T$ is consistent using the compactness theorem. Let $T'$ be a finite subset of $T$. We fix finitely many strongly connected components $X_0,\ldots, X_n$ and elements $a_i\in X_i$ for $i\leq n$ so that the set of $a$ so that $P_a$ or $Q_a$ are mentioned in $T'$ is a subset of $D=\bigcup_{i\leq n} Z(a_i)$.

    We consider the AF $F\vert_D$ and we consider each $X_i$ successively, beginning with an initial one in the ordering of strongly connected components. Since $X_i$ is initial, $T'$ does not contain a formula $\psi_{a,Y}$ for any $a\in X_i$. We take any stage extension on $F\vert_{X_i\cap D}$ and put these elements into $S$. We then successively consider the other $X_j$ where $X_j$ is initial among the strongly connected components not yet considered. Now it is possible that we have formulas $\psi_{a,Y}$ in $T'$ for $D\cap X_i$. These formulas tell us exactly which elements we are to put into $S^\oplus$, if we have determined $Y$ to be the set of members of $S\cap (Z(a)\cup Z(a)^-)^-$ on higher levels. Note that if $a<b$ are in the same $\SCC$ and $Y=Y'\cap (Z(a)\cup Z(a)^-)^-$, then $Z^a_Y\subseteq Z^b_{Y'}$. Let $K=\{a\mid \exists Y (\psi_{a,Y}\in T'\wedge Y\text{agrees with $S$})\}$. Agreeing with $S$ means that $Y$ is exactly the intersection of the set $S$ (which we have already built on higher SCCs) intersected with $(Z(a)\cup Z(a)^-)^-\smallsetminus \SCC(a)$. Then we only need to realize the single axiom with $a\in K$ maximal. 
    
    But for $\psi_{a,Y}$ to be included in $T$, there must be some conflict-free extension $\hat{S}$ in $X_j$ so that $(Y\cup \hat{S})^\oplus$ includes $Z_Y$. We let $S$ on $X_j\cap D$ agree with this $\hat{S}$. Proceeding as such through all the $X_k$ for $k\leq n$, we build a set $S$ so that $\{P_a \mid a\in S\}\cup \{Q_b \mid b\in S^\oplus\}$ is a model of $T'$. This shows that $T$ is consistent.

    By the compactness theorem for propositional logic, $T$ is consistent, so let $\pi$ be a model of $T$. Let $S$ be the set of $a$ so that $\pi$ makes $P_a$ true. It follows from $T_1$ that $S^\oplus$ is the set of $b$ so that $\pi$ makes $Q_b$ true. We need only check that $S$ is a \stgo-extension of $\F$. $S$ is conflict-free by $T_0$. Let $X$ be any strongly connected component. Suppose towards a contradiction that $S'\subseteq X\smallsetminus D_{S}(X)$ were a conflict-free extension so that $(S')^\oplus$ properly contains $S^\oplus$. Let $a$ be the $<$-first element of $X$ so that $a\in (S')^\oplus\smallsetminus S^\oplus$. Let $Y$ be the set of elements in $(Z(a)\cup Z(a)^-)^-\smallsetminus X$ which are in $S$. Then the formula $\psi_{a,Y}$ being in $T$ implies that $S^\oplus$ is the lexicographically-greatest possible subset of $D(a)$, contradicting the existence of $S'$. Thus $S\cap  (X\smallsetminus D_{S}(X))$ is a stage extension of $X\smallsetminus D_{S}(X)$, showing that $S$ is an \stgo-extension of $F$.
\end{proof}

\cfostgoFinitaryDirectionality*
\begin{proof}
    Let $\F=(A_\F,R_\F)$ be an AF and $U\subseteq A_\F$ be unattacked from outside, so $U$ is a union of an initial set of SCCs. If $X\in \cfo(\F)$ (or $X\in \stgo(\F)$), then it is immediate that $X\cap U\in \cfo(\F\restriction_U)$ (or $X\cap U\in \stgo(\F\restriction_U)$).

    Let $S\in \cfo(\F\restriction_U)$. We use finitary extistence applied to the AF $G=\F\restriction_{A_\F\smallsetminus (U\cup S^+)}$, to take an $S'\in \cfo(G)$. We next will show that $S\cup S'\in \cfo(\F)$. Once we see this, then $S\in \{X\cap U \mid X\in \cfo(\F)\}$ completing the proof that \cfo has finitary directionality.

    Towards showing that $S\cup S'\in \cfo(\F)$, let $X$ be an SCC in $A_\F\smallsetminus U$.
    Then the elements of $X$ partition into the elements of $S^+$ and the elements in $\bigcup_{Y\in \Theta} Y$ for some set $\Theta\subseteq \SCC(G)$. For each $a\in Y\in \Theta$, either $a$ is attacked by some element of $S'$ or it is in conflict with an element of $Y\cap S'$. Thus every element of $X$ is either attacked by $S\cup S'$ or is in conflict with an element of $(S\cup S')\cap X$. It follows that $(S\cup S')\cap X$ is a naive extension in $\F\restriction_{X\smallsetminus D_{S\cup S'}(X)}$. Thus $S\cup S'\in \cfo(\F)$. 

    Next we shift to showing finitary directionality of \stgo. Let $S\in \cfo(\F\restriction_U)$. As in the proof of Theorem \ref{thm:finitary existence stg}, we will form a propositional theory in the language $\{P_a,Q_a \mid a\in A_\F\}$. Let $T_0$ be the formulas $\neg(P_a\wedge P_b)$ for all pairs $a\att b$. Let $T_1$ be the formula $Q_a\leftrightarrow P_a\vee \bigvee_{b\att a}P_b$. Let $T_2$ be the set of formulas $\psi_{a,Y}$ for each $a\in A_\F\smallsetminus U$, and $Y\subseteq (Z(a)\cup Z(a)^-)^-\smallsetminus \SCC(a)$. Recall the definition of the formula $\psi_{a,Y}$ from Theorem \ref{thm:finitary existence stg}. Lastly, let $T_3$ be the set of axioms $P_a$ for each $a\in S\cap U$, and $\neg P_a$ for each $a\in U\smallsetminus S$. Let $T=T_0\cup T_1\cup T_2\cup T_3$. A similar compactness argument as in Theorem \ref{thm:finitary existence stg} shows that $T$ is consistent. In particular, we will again need to consider finitely many elements from finitely many SCCs. Then, successively considering initial SCCs $X$, we take extensions $S'$ which make $(Y\cup S')^\oplus$ lexicographically-maximal, where $Y$ is the set of elements already in $S'$ from higher SCCs (including from $U$). As in Theorem \ref{thm:finitary existence stg}, we conclude the consistency of $T$.
    A model $\pi$ of $T$ gives a \stgo extension $X$ of $\F$ so that $S=X\cap U$. 
\end{proof}

\section{A new proof that every finitary AF has a stg extension}

We prove in this section that every finitary AF has a stg extension. This result already appeared \cite{BS-unrestricted}, but we believe that this proof offers a new perspective. It is motivated by the method used to prove Theorems \ref{thm:finitary existence cf2} and \ref{thm:finitary existence stg}.

\begin{theorem}
    Let $\F$ be a finitary AF. Then $\F$ has a stage extension.
\end{theorem}
\begin{proof}
    We will build a stage extension $S$ in two passes. First we will decide which elements should be in $S^\oplus$, ensuring that this is a maximal set. Then we will find a set $S$ which puts each of these elements in $S^\oplus$.

    Fix an ordering $<$ on $\F$ as in Lemma \ref{lem:orderingF}. We will construct a set $D\subseteq A_\F$ by recursively (in $<$) deciding for each $a\in \A_\F$ whether or not to put $a$ into $D$. We consider each element in order. When we consider $a$, we put $a$ into $D$ if and only if there exists some set $B$ which is conflict-free so that $(D\cap \{b \mid b<a\}\cup \{a\})\subseteq B^\oplus$. 
    It is immediate from the definition that there is no set $B$ so that $B^\oplus\supsetneq D$. Thus any conflict-free set $S$ so that $S^\oplus=D$ is a stage extension.

    We construct a propositional theory $T$ in the language with predicates $P_a$ and $Q_a$ for each $a\in \A_\F$. We interpret $P_a$ as saying that $a\in S$ and $Q_a$ as saying that $a\in S^\oplus$.

    Let $T_0$ be the theory containing the sentences $\neg (P_a\wedge P_b)$ for each pair $a\att b$. Let $T_1$ be the theory containing the sentences $Q_a\leftrightarrow P_a\vee \bigvee_{b\att a} P_b$. Let $T_2$ be the theory containing the sentences $Q_a$ for each $a\in D$. Let $T=T_0\cup T_1\cup T_2$. In particular, $T$ says that $S$ is a conflict-free extension so that $S^\oplus$ contains $D$.

    We will use compactness to show that $T$ is consistent. Let $T'$ be a finite subset of $T$. Let $X$ be the set of elements $a$ so that $P_a$ or $Q_a$ is mentioned in $T'$. Let $G$ be the AF $F\vert_{X\cup X^-}$. We argue that $G$ has a conflict-free extension $S$ so that $D\cap X\subseteq S^\oplus$. Let $b$ be the $<$-largest element of $D\cap X$. Then to put $b$ into $D$, we saw that there exists some set $S_b\subseteq \A_\F$ which is conflict-free and $(D\cap \{a \mid a<b\})\cup \{b\}\in S_b^\oplus$. Let $S=S_b\cap (X\cup X^-)$. Then $S$ is conflict-free. For any $x\in D\cap X$, there must be some $y\in S_0$ so that $y=x$ or $y\att x$. In the latter case, $y\in S_b\cap X^-\subseteq S$. Thus $D\cap X\subseteq S^\oplus$. It follows that $T'$ is a consistent sub-theory of $T$.

    By compactness, $T$ is a consistent theory, so we let $\pi$ be a model of $T$. Then define $S$ to be the set of $a$ so that $\pi$ makes $P_a$ true. Then $S$ is conflict-free by $T_0$. $S^\oplus$ contains $D$ by $T_1\cup T_2$. Thus $S$ is a stage extension of $\F$.
\end{proof}

\end{document}